\PassOptionsToPackage{numbers,sort&compress}{natbib}
\documentclass{article}
\usepackage{myarxiv}
\usepackage{amsthm}
\usepackage{amsmath}

\DeclareMathOperator*{\argmax}{arg\,max}

\theoremstyle{definition}

\newtheorem{theorem}{Theorem}
\newtheorem{definition}[theorem]{Definition}
\newtheorem{proposition}[theorem]{Proposition}

\newtheorem{lemma}[theorem]{Lemma}
\usepackage{subcaption}
\usepackage{graphicx}
\usepackage{booktabs}
\usepackage{float}
\usepackage{tabularx}
\usepackage[utf8]{inputenc}
\usepackage[T1]{fontenc}
\usepackage{hyperref}
\usepackage{url}
\usepackage{bbm}
\usepackage{booktabs}
\usepackage{amsfonts}
\usepackage{nicefrac}
\usepackage{microtype}
\usepackage{xcolor}

\usepackage{enumitem}

\title{Stress-Testing Neural Network Verifiers with Provably Robust Instances}

\author{
David Troxell$^{1}$ \And
Yulia Alexandr$^{2}$ \And
Sofia Hunt$^{1}$ \And
Stephanie Lei$^{1}$ \And
Guido Mont\'ufar$^{1,2,3}$ \\
\\
$^1$Department of Statistics \& Data Science, University of California, Los Angeles \\
$^2$Department of Mathematics, University of California, Los Angeles \\
$^3$Max Planck Institute for Mathematics in the Sciences, Leipzig}
\begin{document}

\maketitle

\begin{abstract}
Neural network verifiers aim to provide formal guarantees on model behavior, but existing verification benchmarks are fundamentally limited by their lack of ground-truth labels. As a result, verifier evaluation relies on indirect heuristics, which prevents exact scoring and systematic study of verifier failure modes. We address this gap by introducing a reusable framework for generating verification instances whose ground-truth robustness labels are known a priori through analytic construction. Our framework led to the discovery of multiple numeric tolerance concerns and an implementation bug in popular verifiers, highlighting the need for ground-truth labels. Additionally, to systematically study verifier failure modes, we introduce the verification Difficulty Profile, a collection of estimable quantities capturing distinct sources of instance hardness. Using our framework and these profiles, we evaluate five state-of-the-art verifiers and show that different instances stress distinct aspects of the verification pipeline. We show that these results can aid the future development of verifiers as they provide actionable targets for improving numerical reliability, relaxation quality, and search behavior. Our code is publicly available: \url{https://github.com/dtroxell19/VeriStressGT.git}.
\end{abstract}

\section{Introduction}
Deep learning systems are rapidly being integrated into safety-critical applications in society, including AI-enabled medical devices, self-driving vehicles, power grid balancing for energy systems, and quality control for chemical processes \citep{stanford_ai_index_2025, ozcanli2020deep, yu2023challenges}. 
However, these systems often exhibit fragile and unpredictable behavior as small input perturbations can induce large and unintended changes in outputs \citep{drenkow2022systematicreviewrobustnessdeep, Meng_2024, LIU2023175}. 
In applications where unexpected fluctuations in model behavior have dire consequences, such instability raises a fundamental question of how to guarantee that a trained model behaves safely under all admissible inputs.

Neural network verification aims to answer this question by providing formal guarantees on model behavior. Given a trained model, an admissible input set, and a desired property, verifier systems attempt to certify whether the property holds for all inputs in the set. 
This work focuses on robustness specifications, where the goal is to certify stable model behavior under bounded input perturbations. 
A wide array of verification techniques have been developed in recent years, including Satisfiability Modulo Theory (SMT) and Mixed Integer Programming (MIP) approaches \citep{katz2017reluplex, tjeng2017evaluating}, convex relaxation methods \citep{wong2018scaling}, and scalable branch-and-bound frameworks such as $\alpha,\beta$-CROWN \citep{wang2021betacrownefficientboundpropagation, xu2020automatic}. 
Modern verifiers can now handle networks with millions of parameters and increasingly complex architectures. 

However, despite algorithmic progress, the evaluation of neural network verifiers themselves remains fundamentally limited. Specifically, existing benchmarks generally lack ground-truth labels indicating whether a verification instance is truly robust or non-robust, and they offer little characterization of instance difficulty beyond verifier runtime. These limitations have three primary, direct consequences: 

First, without ground-truth labels, verifier evaluation currently relies on indirect heuristics. For example, VNN-COMP, the annual competition for neural network verification, uses majority voting when verifiers disagree on 
a model instance due to numerical tolerance issues~\citep{kaulen20256th}. Additionally, when no verifier finds a valid robustness counterexample and instead claims robustness or times out, the instance is treated as robust under an assumption of verifier soundness~\citep{kaulen20256th}. 
However, this assumption may be more fragile than anticipated as recent works have found bugs and imprecisions in various verifier implementations \citep{zhou2024soundnessbench,DBLP:journals/corr/abs-2003-03021,zombori2021fooling}. 

Second, without ground-truth labels, benchmarks cannot systematically stress-test the ability of verifiers to certify increasingly difficult robust instances. Non-robust instances are often quickly exposed by adversarial search methods such as projected gradient descent, making them relatively easier to classify \cite{madry2019deeplearningmodelsresistant}. The harder and more informative regime is robust-but-difficult instances, where a verifier must prove that no counterexample exists. Recent work makes existing benchmarks harder while preserving their original ground truth by destabilizing ReLU activations~\citep{relusplitter}. However, this approach inherits labels from the input benchmark and only varies one difficulty mechanism. Benchmarks therefore still lack a general way to generate robust instances with verifier-independent labels and varied, targeted knobs for different bottlenecks.


Third, current verifier benchmark outcomes alone often provide limited diagnostic information about why a verifier succeeds, fails, or times out. Prior work has shown that performance is related to coarse factors such as network architecture and input dimension~\citep{GDVB}. However, these high-level factors provide less insight into whether a verifier is limited by loose relaxations, complex local geometry, or other bottlenecks. Without targeted instances that isolate such bottlenecks, it is also difficult to evaluate verifiers designed to address a specific failure mode.

In this work, we address these gaps by introducing VeriStress-GT (Verifier Stress-Testing via Ground Truth), a modular and re-usable framework to generate benchmark instances for neural network verification with analytically proven ground-truth robustness labels. Moreover, we design these constructions to have controllable difficulty, enabling gradual systematic stress-testing of verification methods for the first time. To guide this process and better analyze verifier failure modes, we introduce the Difficulty Profile, a collection of estimable quantities that characterize verification instance hardness. More broadly, this work reframes verification benchmarking from heuristic evaluation to controlled, ground-truth experimentation. Our primary contributions are as follows: 

\begin{itemize}[leftmargin=*]

    \item We introduce VeriStress-GT, to our knowledge the first reusable framework for generating neural network verification instances with verifier-independent ground-truth robustness labels and controllable difficulty for systematic verifier stress-testing. 
    
    \item We propose the verification instance Difficulty Profile, a unified framework for characterizing verification instance hardness. We show that the profile components enable the systematic evaluation of verifier scalability and failure modes. 
    
    \item We evaluate five state-of-the-art verifiers using VeriStress-GT and Difficulty Profiles, revealing distinct limitations across verification paradigms and identifying previously unreported numerical and implementation issues. 
\end{itemize}

\section{Robust Constructions for Neural Network Verification Benchmarking} 
\label{sec:robust_constructions} 

We first define the robustness verification instance as the tuple $(f, x_0, y, \mathcal{B}_\epsilon)$ consisting of a trained neural network $f : \mathbb{R}^d \to \mathbb{R}^c$, a nominal input $x_0 \in \mathbb{R}^d$ with correct class $y \in [c]$, and a perturbation set $\mathcal{B}_\epsilon(x_0) = \{x : \|x - x_0\|_p \leq \epsilon\}$. In this work we specifically study verification for network robustness; in other words, does $f$ classify every input in $\mathcal{B}_\epsilon(x_0)$ identically to $x_0$? 
Define the per-class margin functions $\mu_k(x) = f_y(x) - f_k(x)$ for $k \in [c]$, $k \neq y$ and the overall margin $\mu(x) = \min\limits_{k \neq y} \mu_k(x)$. 
Then, we state that $f$ is provably robust at $x_0$ if and only if $\min\limits_{x \in \mathcal{B}_\epsilon(x_0)} \mu(x) > 0$. 

Rather than defining a fixed set of benchmark instances, VeriStress-GT defines a novel collection of constructors, each of which produces verification instances with analytically proven robustness labels. 
Each constructor targets a particular architectural class or verifier bottleneck, allowing the benchmark to probe different sources of verification difficulty. 
Figure~\ref{fig:bench_overview} provides a high-level overview of the constructors and framework design principles. 
We briefly describe each of the constructors within the framework below and include constructor details in Appendix~\ref{app:constructor_details}. 

\begin{figure}[t]
    \centering
    \includegraphics[width=.82\linewidth]{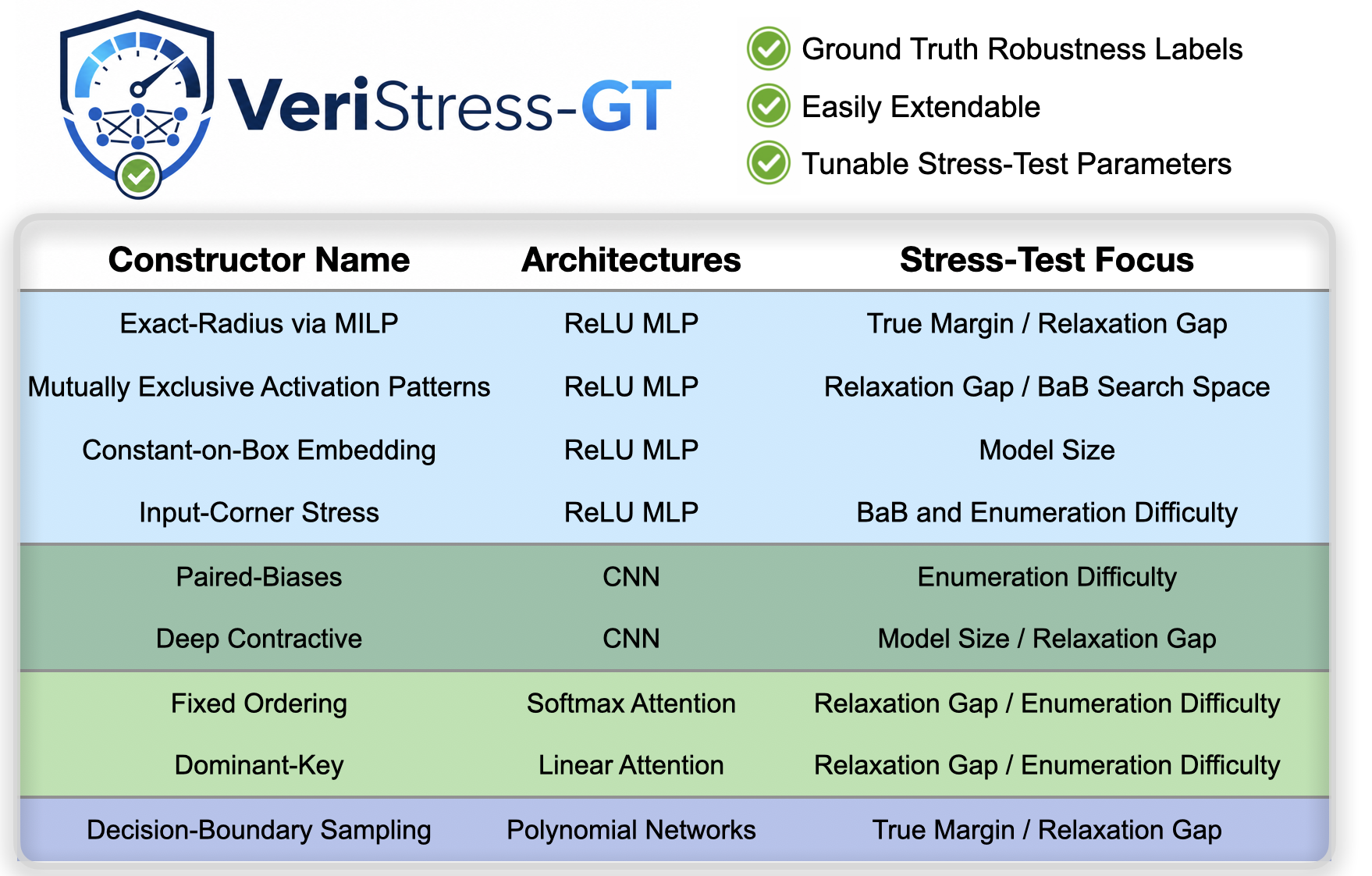}
    \caption{High-level overview of VeriStress-GT, a benchmark for stress-testing neural network verifiers. Benchmark instances are generated via various constructors, or methods to generate provably robust instances. New constructors can easily be added to the framework, and each constructor admits difficulty parameters that can gradually increase verification difficulty while retaining robustness.}
    \label{fig:bench_overview}
\end{figure}

\subsection{ReLU Network Constructors}

Throughout this section, let $\sigma(t)=\max\{t,0\}$ denote the ReLU activation function applied elementwise. A ReLU MLP is a composition of affine maps and ReLU, producing logits $f_1(x),\dots,f_c(x)$. 

\textbf{Exact-Radius via Mixed-Integer Linear Programming (MILP).} 
We generate near-boundary robust instances by computing the network's exact robustness radius using a standard mixed-integer linear programming (MILP) encoding of ReLU networks \citep{DBLP:journals/corr/abs-1711-07356}. Specifically, for each target class $k\neq y$, we solve a mixed-integer linear program that minimizes the $\ell_\infty$ perturbation such that $f_k(x)\geq f_y(x)$. 
Let $t_k^\star$ be the optimal perturbation for class $k$, and define $r^\star=\min_{k\neq y}t_k^\star$. Then any perturbation radius $\epsilon<r^\star$ gives a robust instance, while choosing $\epsilon$ close to $r^\star$ produces an instance with little robustness slack. The full MILP formulation is provided in Appendix~\ref{app:milp}. 

\textbf{Constant-on-Box Embedding.} 
Given an intermediate latent dimension $d_z$ and two neural networks $\Pi: \mathbb{R}^d \mapsto \mathbb{R}^{d_z}$, $\psi : \mathbb{R}^{d_z} \mapsto \mathbb{R}^c$, we consider the composition 
$f = \psi(\Pi(x))$. 
If $\Pi(x)=t \; \forall x \in \mathcal{B}_\epsilon(x_0)$ for some constant vector $t\in\mathbb{R}^{d_z}$, then any arbitrarily large, complex downstream $\psi$ is also constant on the perturbation set. Furthermore, if $(\psi(t))_y - (\psi(t))_k\geq\gamma>0$ for all $k\neq y$, then $f$ is robust with margin at least $\gamma$. 
For $\ell_\infty$ perturbations, the set $\mathcal{B}_\epsilon(x_0)$ is an axis-aligned box, so the condition $x\in\mathcal{B}_\epsilon(x_0)$ decomposes into independent interval constraints on each input coordinate. This allows $\Pi$ to be built from one-dimensional ReLU hinge functions applied separately to each coordinate.
To see this in one dimension, choose two scalars $\alpha$ and $\beta$ such that $[x_0-\epsilon,x_0+\epsilon] \subseteq[\alpha,\beta]$. Then the scalar map $x \mapsto t+\sigma(\alpha-x)+\sigma(x-\beta)$ is equal to $t$ throughout the perturbation interval. Applying this coordinatewise gives a ReLU network that collapses the entire input box to a single point in the latent dimension. 
The primary parameter used to stress-test verifiers is therefore the depth and width of the downstream network $\psi$, which can be made arbitrarily large while preserving the same ground-truth robustness label. 

\textbf{Mutually Exclusive Activation Patterns (MEAP).} 
Most relaxation-based verifier implementations initially treat network neurons independently before performing additional computationally intensive refinement. 
Therefore, to stress-test verifiers, MEAP constructs a ReLU network whose hidden units are individually unstable (i.e., switch between active and inactive as $x$ varies within $\mathcal{B}_\epsilon(x_0)$) but are coupled by 
geometric constraints. Specifically, neurons are organized into pairs such that, for every input in the perturbation set, each pair contains at least one active neuron. 
To embed this geometric constraint, the network first forms $P$ pairs of affine preactivations. 
For each $p\in[P]$, let $z_{p,1}(x)=w_p^\top x+b_{p,1}$ and $z_{p,2}(x)=-w_p^\top x+b_{p,2}$. 
Each pair is then aggregated by $r_p(x)=\max\{\sigma(z_{p,1}(x)),\sigma(z_{p,2}(x))\}$, and the target logit is defined as the minimum over pair scores: 
\begin{equation}
    f_y(x)
    =
    \min_{p\in[P]} r_p(x),
    \qquad
    f_k(x)=0\quad \text{for all } k\neq y.
    \label{eq:meap_network}
\end{equation}
This function is piecewise linear and can be implemented as a ReLU network using standard identities for pairwise maxima and minima. 
The robustness follows from choosing biases $b_{p,1}$ and $b_{p,2}$ such that, for every $x\in\mathcal{B}_\epsilon(x_0)$, at least one preactivation in each pair is positive. 
Equivalently, each pair satisfies $r_p(x)>0$ throughout the perturbation set. 
Hence $f_y(x)>0=f_k(x)$ for all $k\neq y$, certifying robustness. Appendix~\ref{app:meap} gives the full construction and proof. 

\textbf{Input-Corner Stress.}
The input-corner constructor creates ReLU networks whose robustness certificate is given by finite corner evaluation. 
Let $\mathcal B_\epsilon(x_0)$ be an $\ell_\infty$ box, suppose the 
logits depend only on an active set 
$A\subseteq [d]$ 
of $d_{act}$ input coordinates. 
We set the correct-class logit to be constant, $f_y(x)=0$, and define each competitor logit as $f_k(x)=h_k(x)-\beta_k$, where $h_k$ is convex on the active-coordinate box. 
Then the per-class margin $\mu_k(x)=f_y(x)-f_k(x)=\beta_k-h_k(x)$ is concave, so its minimum over the box is attained at an extreme point. Let 
$\mathcal V_\epsilon(x_0;A)=\{x_0+\epsilon s:s_i\in\{-1,+1\}\ \text{for } i\in A, s_i=0\ \text{for } i\notin A\}$ denote the active-coordinate corners. 
Choosing $\beta_k=\max_{v\in\mathcal V_\epsilon(x_0;A)} h_k(v)+\gamma \label{eq:corner_beta}$ 
implies $\mu_k(x)\geq\gamma$ for all $x\in\mathcal B_\epsilon(x_0)$, so the instance is robust with margin at least $\gamma$. 
A ReLU implementation takes $h_k(x)=\sum_{j=1}^J c_{k,j}\sigma(a_j^\top x+b_j)$ with $c_{k,j}\geq 0$, since nonnegative sums of ReLU hinges are convex. 
The construction can nevertheless be difficult for standard verifiers as many hinges can be placed near instability at $x_0$, making local relaxations loose, while the true proof relies on the 
fact that convex functions attain their maximum over a box at a corner. 
The main difficulty parameters are the number of hinges $J$, active dimension $d_{act}$, hinge scale $\|a_j\|_1$, and margin 
$\gamma$. 
Appendix~\ref{app:corner} gives the formal corner certificate and implementation details. 

\subsection{CNN Constructors} 

For CNN constructors, inputs are tensors $x\in\mathbb R^{C_{\mathrm{in}}\times H_{\mathrm{in}}\times W_{\mathrm{in}}}$. 
We write $Z$ for an intermediate feature map and use $\star$ to denote convolution with 
padding and stride fixed by the architecture. 
A convolutional layer with filters $\{\mathcal{W}_i\}_{i=1}^K$ and biases $\{b_i\}_{i=1}^K$ produces preactivations $u_{hw}^{(i)}=(\mathcal{W}_i\star Z)_{hw}+b_i$ at output channel $i$ and spatial location $(h,w)$, followed by a 
coordinatewise activation such as ReLU. 

\textbf{Deep Contractive CNN.}
This constructor attains robustness by forcing perturbations to contract through depth. 
We write the feature extractor as $\Phi = C_D\circ C_{D-1}\circ\cdots\circ C_1\circ P$, where $P$ is an optional front-end map and $C_1,\dots,C_D$ are contractive convolutional blocks. 
Let $L_{\mathrm{front}}$ denote the 
$\ell_\infty$ Lipschitz constant of $P$, with $L_{\mathrm{front}}=1$ if no front-end map is used. 
Each block $C_\ell$ consists of a convolutional layer followed by a $1$-Lipschitz pointwise activation function and is normalized to have induced $\ell_\infty$ Lipschitz constant of at most $\lambda<1$. Intuitively, then, each layer in $P$ may expand perturbations by at most $L_{\mathrm{front}}$ while each subsequent block shrinks the worst-case perturbations by a factor of at most $\lambda$. 
The formal Lipschitz certificate in Appendix~\ref{app:contractive_cnn} shows that: 
\begin{equation}
    \|\Phi(x)-\Phi(x')\|_\infty
    \leq
    L_{\mathrm{front}}\lambda^D\|x-x'\|_\infty . 
\end{equation}
We set the correct-class logit to be a large constant plus a centered perturbation term $ f_y(x)=\Gamma+w_y^\top(\Phi(x)-\Phi(x_0))$ and set each incorrect class 
logit equal to $0$. Therefore, over the perturbation set $\mathcal B_\epsilon(x_0)$, $f_y$ can decrease by at most $\|w_y\|_1 L_{\mathrm{front}}\lambda^D\epsilon$. Choosing $\Gamma>\|w_y\|_1 L_{\mathrm{front}}\lambda^D\epsilon$ therefore certifies robustness on $\mathcal B_\epsilon(x_0)$. The main difficulty parameters are the depth $D$, contraction rate $\lambda$, spatial dimension, front-end Lipschitz constant $L_{\mathrm{front}}$, and margin slack $\Gamma-\|w_y\|_1 L_{\mathrm{front}}\lambda^D\epsilon$.

\textbf{Paired-Bias CNN.} 
The paired-bias CNN construction is a convolutional analogue of MEAP since it involves creating pairs of convolutional channels whose ReLU activations may be unstable, but whose difference is globally nonnegative. The construction may be applied after arbitrary upstream convolutional layers since the certificate holds for every possible value of the shared convolutional response. Let $Z=\Psi(x)$ denote the feature map produced by this upstream network. For each filter pair $i\in[P]$, two channels share the same convolutional filter $\mathcal{W}_i$ but have ordered biases $b_i>c_i$. 
At the layer where the construction is applied, let $s^{(i)}=\mathcal W_i\star Z\in\mathbb R^{H_{\mathrm{sp}}\times W_{\mathrm{sp}}}$ denote the shared convolutional response for pair $i$, where $H_{\mathrm{sp}}$ and $W_{\mathrm{sp}}$ are the height and width of the feature map. Therefore, at spatial location $(h,w)$, the paired ReLU inputs are $s_{hw}^{(i)}+b_i$ and $s_{hw}^{(i)}+c_i$. 
Since the ReLU function is monotone and $b_i>c_i$, we have that $\sigma(s+b_i)-\sigma(s+c_i)\geq 0 \; \forall s\in\mathbb R$.
The target logit averages these nonnegative differences over all pairs and spatial locations: 
\begin{equation}
    f_y(x)=\Gamma+\frac{1}{PH_{\mathrm{sp}}W_{\mathrm{sp}}}\sum_{i=1}^P\sum_{h,w}
    \left[ \sigma(s_{hw}^{(i)}+b_i)-\sigma(s_{hw}^{(i)}+c_i) \right],
    \qquad
    f_k(x)=0\quad \text{for all }k\neq y.
    \label{eq:paired_bias_logit}
\end{equation}
Thus $f_y(x)\geq \Gamma>0$ for all inputs, giving a global robustness certificate. The construction becomes difficult when both channels in many pairs are unstable over the perturbation set. A relaxation that treats $\sigma(s+b_i)$ and $\sigma(s+c_i)$ independently may produce a negative lower bound on their difference, even though the true difference is always nonnegative because the two ReLUs share the same scalar $s$. The number of pairs $P$, size of feature map $H_{\mathrm{sp}}\times W_{\mathrm{sp}}$, bias gap $b_i-c_i$, and margin constant $\Gamma$ control the accumulated relaxation error. Appendix~\ref{app:paired_bias_cnn} gives the formal robustness proof. 

\subsection{Attention Module Constructors} 

For attention module-based constructors, inputs are sequences $X\in\mathbb R^{n\times d_{\mathrm{tok}}}$, 
with rows $x^{(i)}\in\mathbb R^{d_{\mathrm{tok}}}$ 
as tokens. 
Perturbations are measured in the entrywise norm $\|X-X_0\|_\infty=\max_{i,j}|X_{ij}-(X_0)_{ij}|$. We write $Q(X)=XW_Q$, $K(X)=XW_K$, and $V(X)=XW_V$ for the query, key, and value projections, where $W_Q\in\mathbb R^{d_{\mathrm{tok}}\times d_q}$, $W_K\in\mathbb R^{d_{\mathrm{tok}}\times d_k}$, and $W_V\in\mathbb R^{d_{\mathrm{tok}}\times d_v}$ are learned matrices.

\textbf{Fixed-Ordering Softmax Attention.} 
For single-head softmax attention, robustness can be certified by showing that the score ordering is stable and the resulting attention output cannot drift enough to change the class. 
Define the score kernel $M=
{W_QW_K^\top}/{\sqrt{d_k}}$, and let $S(X)=XMX^\top$ with $S_{ij}(X)=(x^{(i)})^\top Mx^{(j)}$. 
The attention output is $\mathrm{Attn}(X)=A(X)XW_V$, where $A(X)$ is the row-wise softmax of $S(X)$, and the full classifier is $f=h\circ \mathrm{Attn}$ for some downstream head $h$. Let $\pi_i^\star=\argmax_j S_{ij}(X_0)$ denote the nominal row-wise score maximizer in row $i$, assuming it is unique, and define the score gaps $\Delta_{ij}=S_{i,\pi_i^\star}(X_0)-S_{ij}(X_0)>0$ for $j\neq \pi_i^\star$. A bilinear expansion of the score differences, derived in
Appendix~\ref{app:fixed_attention}, gives:
\begin{equation}
    \left|
    \left[S_{i,\pi_i^\star}(X)-S_{ij}(X)\right]-\Delta_{ij}
    \right|
    \leq
    \epsilon C_{ij},
    \label{eq:attention_gap_bound}
\end{equation}
for all $\|X-X_0\|_\infty\leq\epsilon$, where $C_{ij}$ depends on $\|M\|_{\mathrm{op}}$, the nominal token norms, $d_\mathrm{tok}$, and $\epsilon$. 
If $\Delta_{ij}>\epsilon C_{ij}$ for every $i$ and $j\neq\pi_i^\star$, then row-wise score maximizers remain fixed throughout $\mathcal{B}_\epsilon(x_0)$.

Appendix~\ref{app:fixed_attention} also provides a perturbation bound for the attention output: there exists a constant $L_{\mathrm{attn}}$ such that  $\|\mathrm{Attn}(X)-\mathrm{Attn}(X_0)\|_F\leq \sqrt n L_{\mathrm{attn}}\epsilon \; \forall X\in\mathcal B_\epsilon(X_0)$. 
Letting $\mu(X_0)=\min_{k\neq y}\left(f_y(X_0)-f_k(X_0)\right)$ denote the nominal classification margin,  if each logit of the downstream head $h$ is $L_h$-Lipschitz with respect to the Frobenius norm, 
then each class-wise margin 
can decrease at most by $2L_h\sqrt n L_{\mathrm{attn}}\epsilon$. 
Hence, if $\mu(X_0)>2L_h\sqrt n L_{\mathrm{attn}}\epsilon$,
then $f$ is robust on $\mathcal B_\epsilon(X_0)$. 

\textbf{Dominant-Key Linear Attention.}
This constructor uses linear attention, where the softmax score kernel is replaced by a nonnegative feature-map inner product. Let $q^{(i)}(X)=Q(X)[i,:]$ and $k^{(j)}(X)=K(X)[j,:]$. For a nonnegative feature map $\phi$ with positive row sums, define:
\begin{equation}
    w_{ij}(X)=
    \langle \phi(q^{(i)}(X)),\phi(k^{(j)}(X))\rangle
    \geq 0.
\end{equation}
The normalized attention weights are then $
    \alpha_{ij}(X)
    =
    \frac{w_{ij}(X)}{\sum_{\ell=1}^n w_{i\ell}(X)},
$ and the 
module's output is:
\begin{equation}
    \mathrm{Attn}_{\mathrm{lin}}(X)[i,:]
    =
    \sum_{j=1}^n \alpha_{ij}(X)V(X)[j,:].
    \label{eq:linear_attention_map}
\end{equation}
Instead of certifying softmax score-order stability, this constructor enforces a dominant key 
per query row. Fix a dominant key $j_i^\star$ for each row and suppose that throughout the perturbation set we have: 
\begin{equation}
    w_{i j_i^\star}(X)
    \geq
    \rho_i \sum_{j\neq j_i^\star} w_{ij}(X).
    \label{eq:linear_attention_dominance}
\end{equation}
This condition says that, for query row $i$, the unnormalized attention weight assigned to key $j_i^\star$ is at least $\rho_i$ times the total weight assigned to all other keys. 
Letting $\rho=\min_i\rho_i$, we therefore have that row $i$ assigns at least $\frac{\rho_i}{(1+\rho_i)}$ normalized attention mass to $j_i^\star$. 
Since the attention output is a weighted average of value vectors, the row output remains close to the value vector $V(X)[j_i^\star,:]$ associated with the dominant key. 
The perturbation proof in Appendix~\ref{app:linear_dominance_attention} shows: 
\begin{equation}
    \|\mathrm{Attn}_{\mathrm{lin}}(X)-\mathrm{Attn}_{\mathrm{lin}}(X_0)\|_{2,\infty}
    \leq
    \frac{2}{1+\rho}D_V(X_0)
    +
    \Big(1+\frac{2}{1+\rho}\Big)
    \epsilon\sqrt{d_{\mathrm{tok}}}\|W_V\|_{\mathrm{op}} ,
    \label{eq:linear_attention_bound}
\end{equation} 
where $D_V(X_0)=\max_i\max_{j\neq j_i^\star}\|V(X_0)[j,:]-V(X_0)[j_i^\star,:]\|_2$ is the largest nominal distance, over query rows, between the value vector associated with the dominant key $j_i^\star$ and any 
other. Thus, if $f=h\circ\mathrm{Attn}_{\mathrm{lin}}$, each logit of $h$ is $L_h$-Lipschitz 
in Frobenius norm, and the nominal margin $\mu(X_0)$ is larger than twice the induced logit drift, then the classifier is robust. For this constructor, difficulty can be controlled through the dominance ratio $\rho$, value spread $D_V(X_0)$, and sequence length.

\subsection{Polynomial Network Constructor} 

The polynomial network constructor generates candidate robust instances by explicitly controlling the geometry of the decision boundary. We consider polynomial classifiers, such as feedforward networks with affine layers and monomial activations, so that each logit is a polynomial in the input. In the binary classification case, let $\mu(x)=f_1(x)-f_2(x)$ be the margin polynomial and let $D=\{x:\mu(x)=0\}$ be the decision boundary. Thus $D$ is a real algebraic hypersurface. The constructor samples a smooth boundary point $p\in D$, chooses a norm-adapted normal direction $u(p)$, 
and places the nominal input at $x_0=p+(\epsilon+\delta)u(p)$, where $\delta>0$ is a boundary buffer. The sampled point $p$ is therefore a known boundary point just outside the perturbation set $B_\epsilon(x_0)$. 

The local placement step creates a near-boundary instance, but the ground-truth robustness label comes from a global separation check 
certifying that $B_\epsilon(x_0)\cap D=\emptyset$. 

This rules out 
another branch or component of the 
decision boundary 
entering the perturbation set. 
When the check succeeds, $\mu$ has constant sign on $B_\epsilon(x_0)$, so the instance is robust. 
The buffer $\delta$ controls proximity to failure: 
decreasing $\delta$ makes the perturbation set nearly touch 
the decision boundary, forcing verifiers to certify a small positive separation from a curved polynomial hypersurface. 
The main difficulty parameters are the input dimension, polynomial degree, hidden width, and boundary buffer $\delta$. Appendix~\ref{app:pnn} gives the full construction, boundary-sampling procedure, and nearest-boundary certificate. 

\section{The Verification Instance Difficulty Profile}\label{sect:profiles}

In addition to the VeriStress-GT framework, we move towards a quantifiable characterization of verification instance difficulty via Difficulty Profiles. 
Such quantification allows one to compare instances beyond binary robust and non-robust labels. The components of the Difficulty Profile are chosen to satisfy design principles ensuring computability, applicability, and  interpretability (see Appendix~\ref{app:comp_detail_results} for an explanation of design principles and candidate components tested). 
Based on the taxonomy of verifier certificate styles (Appendix~\ref{app:taxonomy}) and design principles, we define the profile: 

\begin{definition}
\label{def:profile}
The \emph{difficulty profile} of a verification instance $(f, x_0, y, \mathcal{B}_\epsilon)$ is the tuple:
\begin{equation}\label{eq:profile}
    \mathbf{D}(f, x_0, y, \mathcal{B}_\epsilon) = \bigl(\widehat M_{\min},\; G_{\mathrm{IBP}},\; U, \;A_\tau,\; d_{\mathrm{eff}}\bigr) .  
\end{equation}
Each component, defined below, targets a different axis of verification difficulty: $\widehat M_{\min}$ measures the amount of empirical margin slack 
to certify, 
$G_{\mathrm{IBP}}$ measures how much of this slack is lost under a coarse interval relaxation, 
$U$ measures the fraction of nonlinear units whose phase remains ambiguous over the input set, 
$A_\tau$ measures the local regions encountered in the input set, 
and $d_{\mathrm{eff}}$ measures the effective dimensionality of the input directions along which the margin is sensitive. 
\end{definition}

\paragraph{Component 1: Minimum Margin ($\widehat{M}_{\min}$).} 
Proximity to violation is a basic, intuitive source of verification difficulty as instances with small positive margin may require substantial branching or bound refinement. Since $\min_{x\in\mathcal{B}_\epsilon(x_0)}\mu(x)$ is generally nonconvex and intractable, we estimate it via a sampling procedure. Let $\mathcal{S} = \{x_1, \dots, x_N \} \subset \mathcal{B}_\epsilon(x_0)$ denote a set of profiling samples, drawn from a mixture of uniform and boundary-biased samples. Then, we define:
\begin{equation}
    \widehat{M}_{\min} := \min_{x_i \in \mathcal {S}} \mu(x_i) \label{mhat}
\end{equation}
A small, positive $\widehat{M}_{\min}$ suggests that the instance is close to a counterexample, while $\widehat{M}_{\min} < 0$ indicates that 
sampling already found one. 
We note that $\widehat{M}_{\min}$ is not designed to be a certificate of robustness, 
but rather a verifier-agnostic proxy for 
proximity to the decision boundary. 
 
\paragraph{Component 2: IBP Relative Gap ($G_{\mathrm{IBP}})$.} 
Relaxation-based verifiers certify robustness 
via tractable lower bounds on $\mu(x)$. 
We therefore use IBP as an efficient, architecture-agnostic proxy for relaxation looseness. 
Let $L_{\mathrm{IBP}}$ denote the IBP lower bound 
over $\mathcal{B}_\epsilon(x_0)$. 
Define the 
relative gap as: 
\begin{equation}
    G_{\mathrm{IBP}} := \frac{\widehat{M}_{\min} - L_{\mathrm{IBP}}}{|\widehat{M}_{\min}| + \eta} , 
    \label{IBPgap}
\end{equation}
where $\eta >0$ is a small numerical stability constant. 
The numerator estimates the 
margin slack lost by IBP by the best sampled margin, while the denominator normalizes 
by the empirical margin scale. 
Hence larger $G_{\mathrm{IBP}}$ 
indicates increased instance difficulty, particularly for relaxation-based verifiers.

\paragraph{Component 3: Unstable Fraction ($U$).}
The unstable fraction measures how many nonlinear neurons cannot be treated as fixed or locally affine over $\mathcal B_\epsilon(x_0)$. 
For ReLU networks, this 
is the fraction of neurons whose IBP pre-activation interval $[\ell_j,u_j]$ crosses zero, i.e., $\ell_j<0<u_j$. 
For smooth nonlinearities, we use the same principle by marking coordinate $j$ as effectively nonlinear when the activation slope varies sufficiently over its IBP interval: 
$\omega_j :=\sup_{s,t\in[\ell_j,u_j]} |\phi_j'(s)-\phi_j'(t)| > \tau$ for a fixed threshold $\tau>0$. Let $\mathcal J$ denote the set of nonlinear coordinates. We define: 
\begin{equation}
    U :=
    \frac{1}{|\mathcal J|}
    \sum_{j\in\mathcal J}
    \mathbbm{1}_{\{\omega_j>\tau\}}.
\end{equation}
For ReLU neurons, this reduces exactly to the unstable-neuron fraction. Larger $U$ indicates that more units require nontrivial relaxation or case splitting.

\paragraph{Component 4: Local-Region Complexity ($A_\tau$).}
The unstable fraction $U$ measures individually nonlinear coordinates, but 
not how many distinct local behaviors are 
realized 
in $\mathcal B_\epsilon(x_0)$. To capture this, we view local-region complexity as an affine-covering problem. For tolerance $\delta>0$, define:
\begin{equation}\label{eq:Naff}
    N_{\mathrm{aff}}(\delta)
    :=
    \min\Bigl\{
    m:
    \mathcal B_\epsilon(x_0)\subseteq \textstyle\bigcup_{i=1}^m U_i,\;
    \exists (a_i,b_i)
    \text{ s.t. }
    \sup\limits_{x\in U_i}
    |\mu(x)-a_i^\top x-b_i|
    \leq \delta
    \Bigr\},
\end{equation}
where the $U_i$ range over convex subsets of $\mathcal B_\epsilon(x_0)$. Let $L_c:=\sup\limits_{x\in\mathcal B_\epsilon(x_0)}\sup\limits_{\xi\in\partial_c\mu(x)}\|\xi\|_{p,*}$ denote a local Lipschitz constant for the margin over $\mathcal B_\epsilon(x_0)$, where $\partial_c\mu(x)$ is the Clarke subdifferential and $\|\cdot\|_{p,*}$ is dual to $\|\cdot\|_p$. Set $A_\tau^\star:= \log N_{\mathrm{aff}}(\tau\epsilon L_c)$. Therefore, $A_\tau^\star$ counts how many affine functions are needed to approximate the  margin over the perturbation set with resolution $\tau$. For ReLU networks, this corresponds to an effective count of affine regions, while for networks with smooth activation functions it counts the number of local Taylor approximations needed to approximate the curved margin landscape. The following proposition motivates our empirical estimator by showing that affine-cover complexity is governed by how much the normalized gradient can vary over $\mathcal{B}_\epsilon(x_0)$:

\begin{proposition}[Smooth affine-cover bound]\label{prop:beta_bound}
Assume that $\mu(x)$ is differentiable on $\mathcal B_\epsilon(x_0)$ and that its gradient is $\beta$-Lipschitz.
Let $\widetilde\beta:=\frac{\beta\epsilon}{L_c(\mu)}$ be the normalized gradient-variation scale. Then, up to the standard covering-number constant for the chosen norm,
$
   \smash{ A_\tau^\star\leq d\log
    \big( 1+\sqrt{\smash[b]{ {2\widetilde\beta}/{\tau}} } \big)
 }  .
$    
\end{proposition}
See Appendix~\ref{app:Proofs} for the proof. Directly estimating the global gradient-Lipschitz scale $\widetilde\beta$ is conservative in practice, so we instead measure the realized variation in local linear behavior by sampling normalized gradients. For each $x_i\in\mathcal S$, let
$q_\tau(x_i):=\operatorname{Quantize}_\tau(\nabla\mu(x_i)/L_c)$, where $\operatorname{Quantize}_\tau(v)$ denotes coordinatewise quantization of $v$ onto a grid of width $\tau$. We define:
\begin{equation}\label{eq:Atau_empirical}
    A_\tau := \log \left| \{q_\tau(x_i):x_i\in\mathcal S\}\right|.
\end{equation}
$A_\tau$ is a sampled log-count of distinct normalized local affine behaviors, as it estimates the realized local complexity suggested by the affine-cover definition. Larger $A_\tau$ indicates more heterogeneous local behavior, making coarse relaxations, abstractions, or partitions less likely to remain tight.

\paragraph{Component 5: Effective Gradient dimensionality ($d_\mathrm{eff}$).} 
Even when the margin is sensitive to perturbations, the structure of this sensitivity matters. If sensitivity is concentrated in a few input coordinates, splitting or search may isolate it quickly, but if it is spread across many, 
the instance is effectively high-dimensional. Let $\nabla \mu(x)$ denote the input gradient of the margin where it exists, using any measurable subgradient at non-smooth points. Define the average effective gradient dimension:
\begin{equation}
\label{eq:deff}
    d_{\mathrm{eff}}:= \frac{1}{|\mathcal S|} \sum_{x_i\in\mathcal S}\frac{ \|\nabla \mu(x_i)\|_1^2
    }{\|\nabla \mu(x_i)\|_2^2 + \eta} , 
\end{equation}
where $\eta>0$ is a small constant for numerical stability. This quantity is related to the Participation Ratio (PR) for measuring dimensionality \cite{Gao214262}, and is near $1$ when the gradient is concentrated on one coordinate and near $r$ when it is spread evenly across $r$ coordinates. Larger $d_{\mathrm{eff}}$ therefore indicates more distributed margin sensitivity, which is especially relevant for $\ell_\infty$ perturbations since the first-order margin decrease scales with $\epsilon\|\nabla\mu(x)\|_1$.

\section{Numerical Experiments and Verification Study}
\label{sect:exp}

We evaluate $5$ top-scoring verifiers on $225$ VeriStress-GT instances spanning all constructors in Section~\ref{sec:robust_constructions}, plus two external robustness benchmarks: MNIST$\_$fc, with fully connected ReLU networks trained on MNIST~\cite{bak2021secondinternationalverificationneural,726791}, and oval$21$, 
with ReLU 
CNNs trained on MNIST and CIFAR-related 
tasks~\cite{bak2021secondinternationalverificationneural,Krizhevsky2009LearningML}. 
The 
VeriStress-GT instances form a parameter-diverse stress-test suite, though the framework can generate 
unlimited instances. We compute Difficulty Profile estimates for all instances.

All experiments were run on an Intel Xeon Gold $6152$ CPU server with $32$ CPU cores, using per-instance timeouts of $600$ seconds for VeriStress-GT and $360$ seconds for external benchmarks. Appendix~\ref{app:detail_results} gives framework instantiation details and aggregate verifier outcomes. Figure~\ref{fig:veriStress_results} summarizes verifier runtimes, while Figure~\ref{fig:profile_correlations} shows how Difficulty Profile components and interactions discriminate verifier timeouts. We highlight three main takeaways:

\begin{figure}[t]
\centering 
    \includegraphics[width=.82\linewidth]{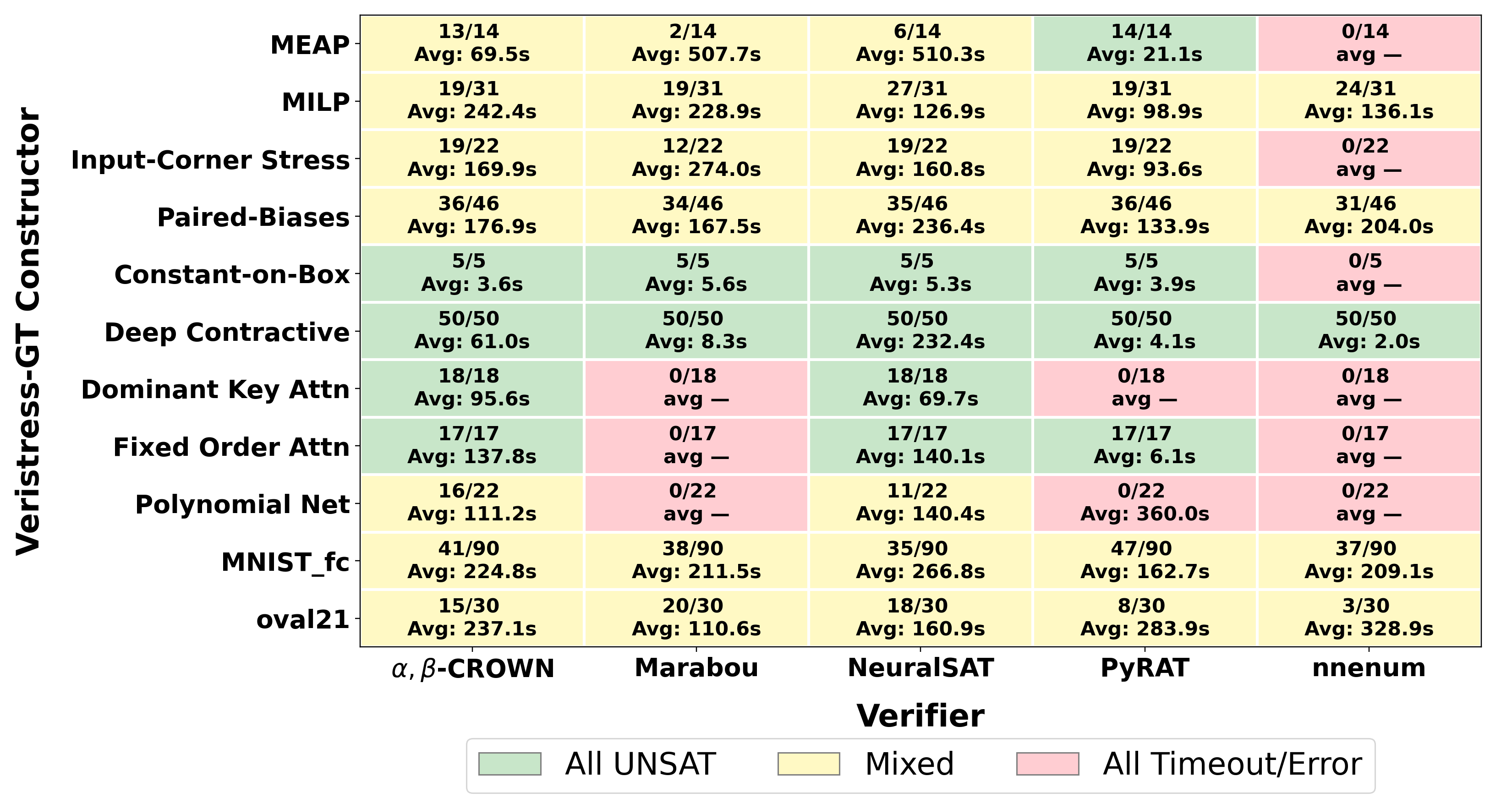}
    \caption{
    Runtime outcomes for $5$ verifiers on VeriStress-GT constructors and two external benchmarks, shown in the last two rows. Average runtimes include timeouts of $600$ seconds for VeriStress-GT and $360$ seconds for external benchmarks. Fractions report correct results over total instances; missing runtimes indicate unsupported computation graphs.
    }
    \label{fig:veriStress_results}
\end{figure}

\textbf{1. Ground-truth labels reveal incorrect robustness claims.} 
Despite all instances in this instantiation of the VeriStress-GT framework containing analytic proofs of UNSAT (i.e., robust) labels, multiple verifiers return false SAT (i.e., non-robust) claims. Many of these incorrect claims can primarily be attributed to numerical tolerance limitations. For instance, setting the perturbation radius $\epsilon$ equal to $0.999$ times the exact robustness radius computed via the MILP approach led to such SAT claims for a number of verifiers. However, when separately testing verifiers with an Exact-Radius MILP construction containing provably non-robust instances (i.e., a perturbation radius $\epsilon$ larger than the network's exact adversarial radius), we discovered a bug in a popular verifier's handling of disjunctive output constraints, leading to false UNSAT claims which have since been corrected after collaboration with developers (see Appendix~\ref{app:NeuralSAT_Bug} for details). Incorrect claims for  verification can be impactful; false UNSAT claims can mislead practitioners into believing a model is certifiably safe when it is not, and false SAT claims distort verifier evaluation. VeriStress-GT helps expose these failures and highlights the need for a careful review of numerical tolerance limits.

\textbf{2. Difficulty Profiles help make verifier timeouts diagnosable.} Without Difficulty Profiles, analysis generally ends at measuring the number of timeouts; some instances are certified, others are not, and the reason largely remains opaque. Difficulty Profiles turn these outcomes into diagnosable failure modes. For example, Figures~\ref{fig:mnist_traj} and~\ref{fig:oval_traj} in Appendix~\ref{app:detail_results} show that the final $16$ MNIST$\_$fc instances have $G_{\mathrm{IBP}}\geq 2\times 10^5$, several orders of magnitude larger than the remaining MNIST$\_$fc instances and the corresponding values in VeriStress-GT or oval$21$, suggesting extreme interval-relaxation looseness as the dominant obstruction. In oval$21$, instances $10$--$16$ are certified by at most two verifiers each and all have $d_{\mathrm{eff}}>1000$, more than twice any value observed in MNIST$\_$fc, suggesting high-dimensional margin sensitivity. These examples illustrate how Difficulty Profiles convert raw verifier outcomes into more actionable explanations, helping identify whether future progress should target tighter relaxations, better branching, dimension-aware heuristics, or other verifier improvements.

\textbf{3. Difficulty Profiles help turn VeriStress-GT into a targeted stress-test suite.} As VeriStress-GT instances scale, different constructors isolate different failure modes, making the benchmark useful not only for measuring performance but also for testing verifier improvements. For MEAP instances, increasing difficulty produces large relaxation gaps $G_{\mathrm{IBP}}$, high instability $U$, and large effective gradient dimensionality $d_{\mathrm{eff}}$. These instances therefore form a targeted stress test for methods that must reason about activation-pattern ambiguity, branching decisions, and loose intermediate bounds. In contrast, the Corners construction has $A_\tau = 0$ and $d_{\mathrm{eff}} \approx 0$, so it is structurally simple from an activation-pattern perspective. However, verifiers slow or timeout as $G_{\mathrm{IBP}}$ grows, showing that relaxation looseness and geometric bound growth can be isolated from nonlinear-region complexity.

These distinctions make the profiles actionable. A new branching heuristic can be evaluated on MEAP instances, a tighter relaxation can be tested on Corners instances, and depth-aware abstractions can be tested on Paired-Bias CNNs. Rather than treating all timeouts as equivalent failures, Difficulty Profiles reveal which instances are useful for stress-testing specific verifier capabilities.

\begin{figure}[t]
    \centering

    \begin{subfigure}[t]{0.49\linewidth}
        \centering
        \includegraphics[width=\linewidth]{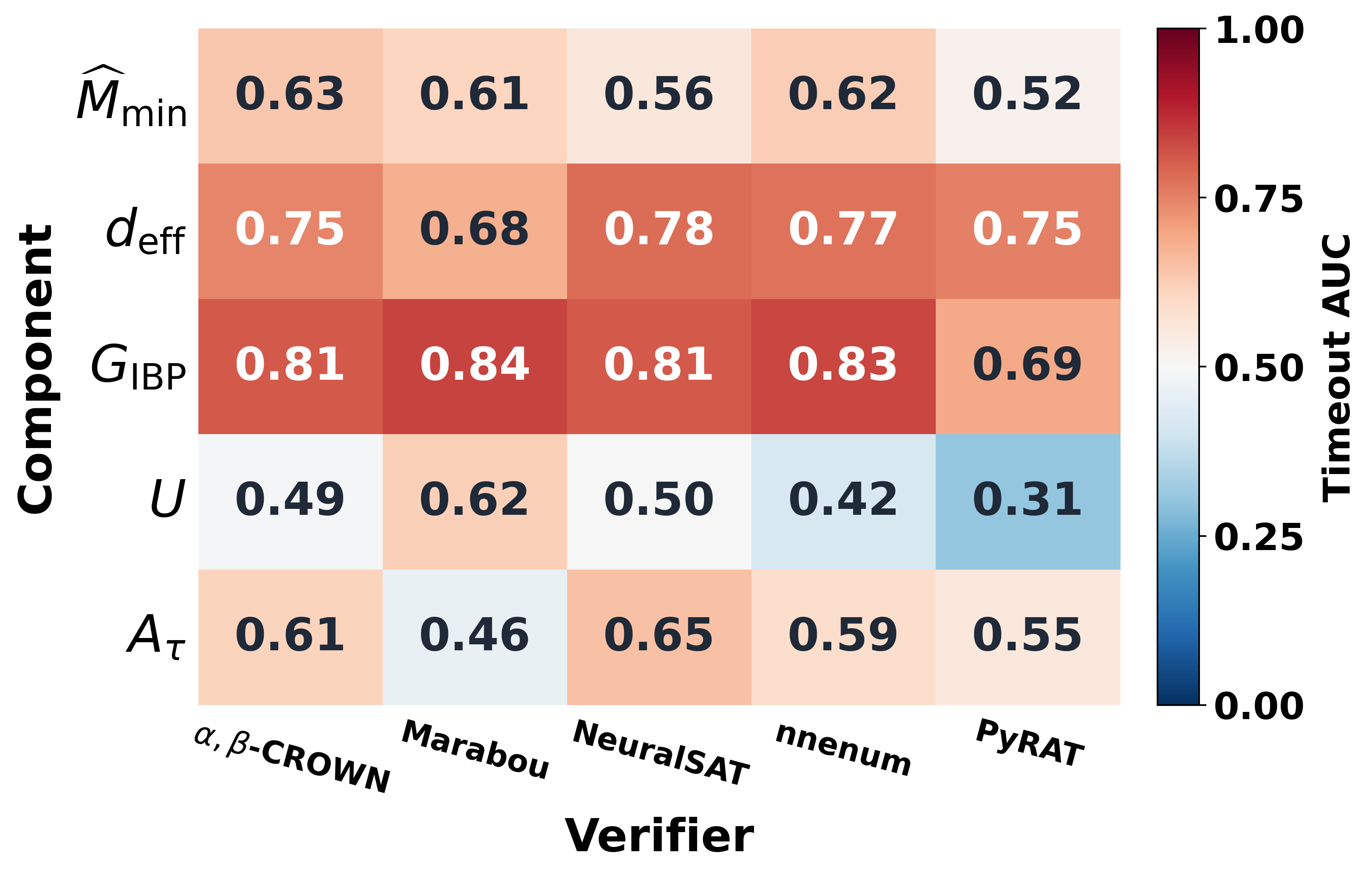}
        \caption{Profile correlations}
        \label{fig:profile_correlations_main}
    \end{subfigure}
    \hfill
    \begin{subfigure}[t]{0.49\linewidth}
        \centering
        \includegraphics[width=\linewidth]{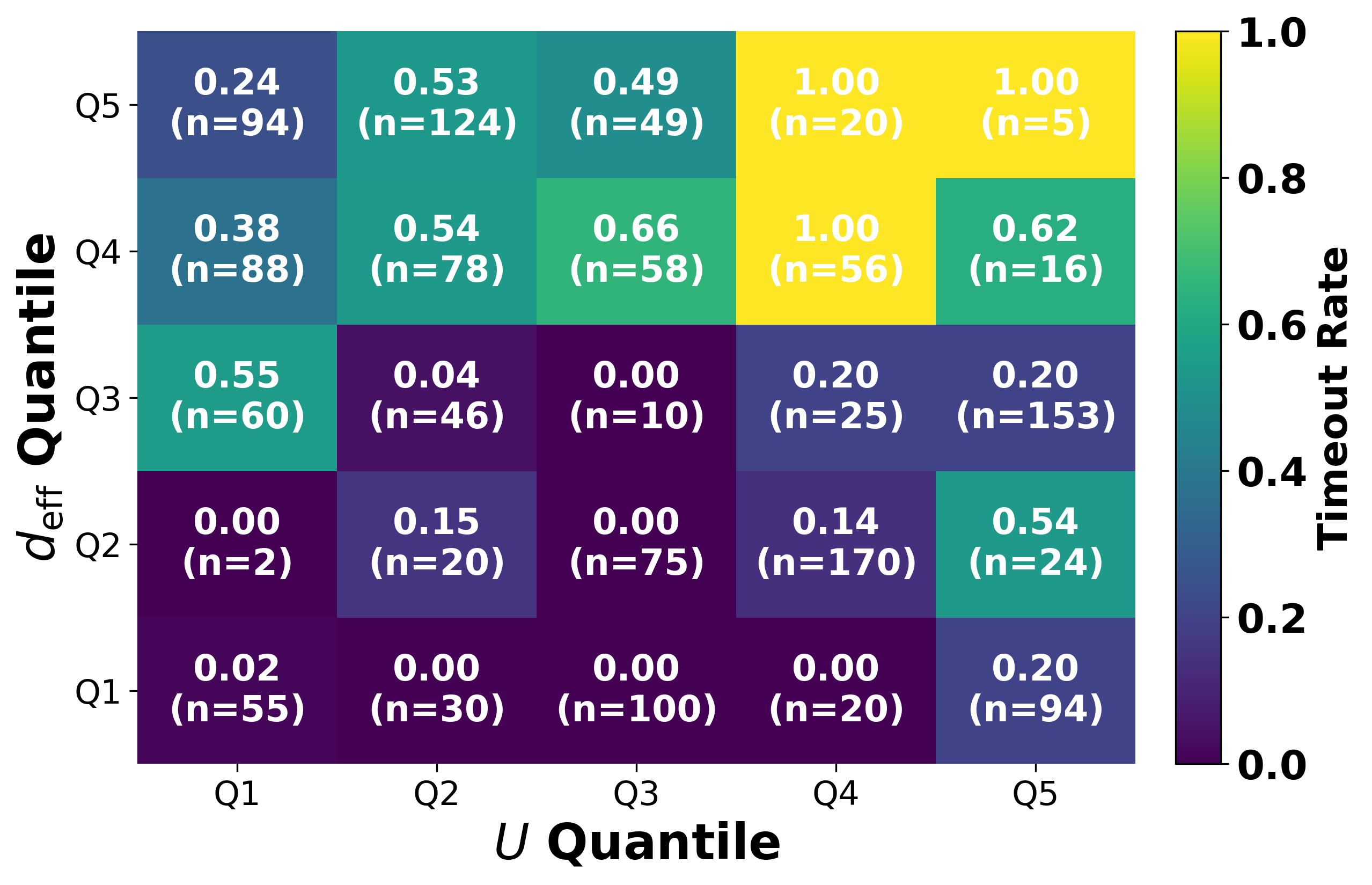}
        \caption{$U \times d_{\mathrm{eff}}$}
        \label{fig:timeout_heatmap_u_deff}
    \end{subfigure}
    \caption{(left) Timeout AUC for each profile component across all benchmarks, treating timed-out instances as positives and solved instances as negatives; values above $0.5$ indicate that larger component values are associated with timeouts.
(right) Binned timeout rate over unstable fraction $U$ and effective gradient dimension $d_{\mathrm{eff}}$, showing that timeouts can depend on interactions between profile components. For example, we see $U$ have little effect on timeout rate when $d_\mathrm{eff}$ is small.
}
\label{fig:profile_correlations}

\vspace{-1em}
\end{figure}

\section{Conclusion, Limitations, and Future Work}
We introduced VeriStress-GT, a modular framework for constructing neural network verification instances with verifier-independent ground-truth robustness labels. Together with the Difficulty Profile, VeriStress-GT enables targeted stress-testing and makes verifier timeouts more diagnosable. It also exposed numerical tolerance issues and implementation bugs, underscoring the need for ground-truth labels in reliable verifier evaluation. 

Several limitations and opportunities for future work remain. 
Due to computational costs, the verifiers were tested with a limited number of repeated seeds. Moreover, we do not claim that Difficulty Profiles are the sole source of predictive power for verifier runtime; instead, we aim for a collection of complementary, succinct, interpretable, and predictive components. 

Potential avenues for future work include extending the framework to additional architectures, such as residual networks and graph neural networks. 
Difficulty Profiles may also help guide and test verifiers for modern language-model components, where component-specific verifiers and compositional certificates could improve scalability. 

\subsection*{Acknowledgments} 
This project has been supported by 
DARPA Artificial Intelligence Quantified (AIQ) grant HR00112520014.  
GM was partially supported by 
NSF grants 
DMS-2522495, 
DMS-2145630, 
CCF-2212520, 
DFG SPP 2298 ``Theoretical Foundations of Deep Learning'' project 464109215,  
and BMFTR in DAAD project 57616814 (SECAI).

\newpage 

\bibliographystyle{plainnat}
\bibliography{references}

@report{stanford_ai_index_2025,
  title   = {{AI} Index Report 2025},
  author  = {{Stanford Institute for Human-Centered Artificial Intelligence}},
  institution = {Stanford University},
  year    = {2025},
  url = {https://hai.stanford.edu/assets/files/hai_ai_index_report_2025.pdf}, 
}

@article{ozcanli2020deep,
  title={Deep learning methods and applications for electrical power systems: A comprehensive review},
  author={Ozcanli, Asiye K and Yaprakdal, Fatma and Baysal, Mustafa},
  journal={International Journal of Energy Research},
  volume={44},
  number={9},
  pages={7136--7157},
  year={2020},
  publisher={Wiley Online Library}
}

@article{yu2023challenges,
  title={Challenges and opportunities of deep learning-based process fault detection and diagnosis: a review},
  author={Yu, Jianbo and Zhang, Yue},
  journal={Neural Computing and Applications},
  volume={35},
  number={1},
  pages={211--252},
  year={2023},
  publisher={Springer}
}

@misc{drenkow2022systematicreviewrobustnessdeep,
      title={A Systematic Review of Robustness in Deep Learning for Computer Vision: Mind the gap?}, 
      author={Nathan Drenkow and Numair Sani and Ilya Shpitser and Mathias Unberath},
      year={2022},
      eprint={2112.00639},
      archivePrefix={arXiv},
      primaryClass={cs.CV},
      url={https://arxiv.org/abs/2112.00639}, 
}

@article{Meng_2024,
   title={Adversarial Robustness of Deep Neural Networks: A Survey from a Formal Verification Perspective},
   noISSN={2160-9209},
   url={http://dx.doi.org/10.1109/TDSC.2022.3179131},
   noDOI={10.1109/tdsc.2022.3179131},
   journal={IEEE Transactions on Dependable and Secure Computing},
   publisher={Institute of Electrical and Electronics Engineers (IEEE)},
   author={Meng, Mark Huasong and Bai, Guangdong and Teo, Sin Gee and Hou, Zhe and Xiao, Yan and Lin, Yun and Dong, Jin Song},
   year={2024},
   pages={1–1} }

@article{LIU2023175,
title = {A comprehensive survey of robust deep learning in computer vision},
journal = {Journal of Automation and Intelligence},
volume = {2},
number = {4},
pages = {175-195},
year = {2023},
noissn = {2949-8554},
nodoi = {https://doi.org/10.1016/j.jai.2023.10.002},
url = {https://www.sciencedirect.com/science/article/pii/S294985542300045X},
author = {Jia Liu and Yaochu Jin}
}

@inproceedings{katz2017reluplex,
  title={Reluplex: An efficient {SMT} solver for verifying deep neural networks},
  author={Katz, Guy and Barrett, Clark and Dill, David L and Julian, Kyle and Kochenderfer, Mykel J},
  booktitle={International conference on computer aided verification},
  pages={97--117},
  year={2017},
  organization={Springer}
}

@misc{zhang2018efficientneuralnetworkrobustness,
      title={Efficient Neural Network Robustness Certification with General Activation Functions}, 
      author={Huan Zhang and Tsui-Wei Weng and Pin-Yu Chen and Cho-Jui Hsieh and Luca Daniel},
      year={2018},
      eprint={1811.00866},
      archivePrefix={arXiv},
      primaryClass={cs.LG},
      url={https://arxiv.org/abs/1811.00866}, 
}

@inproceedings{tjeng2017evaluating,
title={Evaluating Robustness of Neural Networks with Mixed Integer Programming},
author={Vincent Tjeng and Kai Y. Xiao and Russ Tedrake},
booktitle={International Conference on Learning Representations},
year={2019},
url={https://openreview.net/forum?id=HyGIdiRqtm},
}

@inproceedings{wong2018scaling,
 author = {Wong, Eric and Schmidt, Frank and Metzen, Jan Hendrik and Kolter, J. Zico},
 booktitle = {Advances in Neural Information Processing Systems},
 noeditor = {S. Bengio and H. Wallach and H. Larochelle and K. Grauman and N. Cesa-Bianchi and R. Garnett},
 pages = {},
 nopublisher = {Curran Associates, Inc.},
 title = {Scaling provable adversarial defenses},
 url = {https://proceedings.neurips.cc/paper_files/paper/2018/file/358f9e7be09177c17d0d17ff73584307-Paper.pdf},
 volume = {31},
 year = {2018}
}

@inproceedings{wang2021betacrownefficientboundpropagation,
 author = {Wang, Shiqi and Zhang, Huan and Xu, Kaidi and Lin, Xue and Jana, Suman and Hsieh, Cho-Jui and Kolter, J. Zico},
 booktitle = {Advances in Neural Information Processing Systems},
 noeditor = {M. Ranzato and A. Beygelzimer and Y. Dauphin and P.S. Liang and J. Wortman Vaughan},
 pages = {29909--29921},
 nopublisher = {Curran Associates, Inc.},
 title = {Beta-{CROWN}: Efficient Bound Propagation with Per-neuron Split Constraints for Neural Network Robustness Verification},
 url = {https://proceedings.neurips.cc/paper_files/paper/2021/file/fac7fead96dafceaf80c1daffeae82a4-Paper.pdf},
 volume = {34},
 year = {2021}
}

@inproceedings{xu2020automatic,
 author = {Xu, Kaidi and Shi, Zhouxing and Zhang, Huan and Wang, Yihan and Chang, Kai-Wei and Huang, Minlie and Kailkhura, Bhavya and Lin, Xue and Hsieh, Cho-Jui},
 booktitle = {Advances in Neural Information Processing Systems},
 noeditor = {H. Larochelle and M. Ranzato and R. Hadsell and M.F. Balcan and H. Lin},
 pages = {1129--1141},
 nopublisher = {Curran Associates, Inc.},
 title = {Automatic Perturbation Analysis for Scalable Certified Robustness and Beyond},
 url = {https://proceedings.neurips.cc/paper_files/paper/2020/file/0cbc5671ae26f67871cb914d81ef8fc1-Paper.pdf},
 volume = {33},
 year = {2020}
}

@misc{kaulen20256th,
      title={The 6th International Verification of Neural Networks Competition (VNN-COMP 2025): Summary and Results}, 
      author={Konstantin Kaulen and Tobias Ladner and Stanley Bak and Christopher Brix and Hai Duong and Thomas Flinkow and Taylor T. Johnson and Lukas Koller and Edoardo Manino and ThanhVu H Nguyen and Haoze Wu},
      year={2025},
      eprint={2512.19007},
      archivePrefix={arXiv},
      primaryClass={cs.LG},
      url={https://arxiv.org/abs/2512.19007}, 
}

@article{zhou2024soundnessbench,
title={{SoundnessBench}: A Soundness Benchmark for Neural Network Verifiers},
author={Xingjian Zhou and Keyi Shen and Andy Xu and Hongji Xu and Cho-Jui Hsieh and Huan Zhang and Zhouxing Shi},
journal={Transactions on Machine Learning Research},
noissn={2835-8856},
year={2025},
url={https://openreview.net/forum?id=UuYYldVLH3},
note={}
}

@article{DBLP:journals/corr/abs-1711-07356,
  author       = {Vincent Tjeng and
                  Russ Tedrake},
  title        = {Evaluating Robustness of Neural Networks with Mixed Integer Programming},
  journal      = {CoRR},
  volume       = {abs/1711.07356},
  year         = {2017},
  url          = {http://arxiv.org/abs/1711.07356},
  eprinttype   = {arXiv},
  eprint       = {1711.07356},
  bibsource    = {dblp computer science bibliography, https://dblp.org}
}

@article{DBLP:journals/corr/abs-1810-12715,
  author       = {Sven Gowal and
                  Krishnamurthy Dvijotham and
                  Robert Stanforth and
                  Rudy Bunel and
                  Chongli Qin and
                  Jonathan Uesato and
                  Relja Arandjelovic and
                  Timothy A. Mann and
                  Pushmeet Kohli},
  title        = {On the Effectiveness of Interval Bound Propagation for Training Verifiably
                  Robust Models},
  journal      = {arXiv:1810.12715},
  year         = {2018},
  url          = {http://arxiv.org/abs/1810.12715},
  eprinttype   = {arXiv},
  eprint       = {1810.12715},
  bibsource    = {dblp computer science bibliography, https://dblp.org}
}

@misc{xu2021fastcompleteenablingcomplete,
      title={Fast and Complete: Enabling Complete Neural Network Verification with Rapid and Massively Parallel Incomplete Verifiers}, 
      author={Kaidi Xu and Huan Zhang and Shiqi Wang and Yihan Wang and Suman Jana and Xue Lin and Cho-Jui Hsieh},
      year={2021},
      eprint={2011.13824},
      archivePrefix={arXiv},
      primaryClass={cs.AI},
      url={https://arxiv.org/abs/2011.13824}, 
}

@InProceedings{10.1007/978-3-030-25540-4_26,
author="Katz, Guy
and Huang, Derek A.
and Ibeling, Duligur
and Julian, Kyle
and Lazarus, Christopher
and Lim, Rachel
and Shah, Parth
and Thakoor, Shantanu
and Wu, Haoze
and Zelji{\'{c}}, Aleksandar
and Dill, David L.
and Kochenderfer, Mykel J.
and Barrett, Clark",
noeditor="Dillig, Isil
and Tasiran, Serdar",
title="The Marabou Framework for Verification and Analysis of Deep Neural Networks",
booktitle="Computer Aided Verification",
year="2019",
publisher="Springer International Publishing",
address="Cham",
pages="443--452"
}

@InProceedings{10.1007/978-3-540-27813-9_14,
author="Ganzinger, Harald
and Hagen, George
and Nieuwenhuis, Robert
and Oliveras, Albert
and Tinelli, Cesare",
noeditor="Alur, Rajeev
and Peled, Doron A.",
title="{DPLL(T)}: Fast Decision Procedures",
booktitle="Computer Aided Verification",
year="2004",
publisher="Springer Berlin Heidelberg",
address="Berlin, Heidelberg",
pages="175--188"
}

@misc{duong2024dplltframeworkverifyingdeep,
      title={A DPLL(T) Framework for Verifying Deep Neural Networks}, 
      author={Hai Duong and ThanhVu Nguyen and Matthew Dwyer},
      year={2024},
      eprint={2307.10266},
      archivePrefix={arXiv},
      primaryClass={cs.LG},
      url={https://arxiv.org/abs/2307.10266}, 
}

@inproceedings{nnenum,
author = {Bak, Stanley},
title = {nnenum: Verification of {ReLU} Neural Networks with Optimized Abstraction Refinement},
year = {2021},
isbn = {978-3-030-76383-1},
publisher = {Springer-Verlag},
address = {Berlin, Heidelberg},
url = {https://doi.org/10.1007/978-3-030-76384-8_2},
doi = {10.1007/978-3-030-76384-8_2},
booktitle = {NASA Formal Methods: 13th International Symposium, NFM 2021, Virtual Event, May 24–28, 2021, Proceedings},
pages = {19–36},
numpages = {18}
}

@misc{tran2020nnvneuralnetworkverification,
      title={{NNV}: The Neural Network Verification Tool for Deep Neural Networks and Learning-Enabled Cyber-Physical Systems}, 
      author={Hoang-Dung Tran and Xiaodong Yang and Diego Manzanas Lopez and Patrick Musau and Luan Viet Nguyen and Weiming Xiang and Stanley Bak and Taylor T. Johnson},
      year={2020},
      eprint={2004.05519},
      archivePrefix={arXiv},
      primaryClass={eess.SY},
      url={https://arxiv.org/abs/2004.05519}, 
}

@article {Gao214262,
	author = {Gao, Peiran and Trautmann, Eric and Yu, Byron and Santhanam, Gopal and Ryu, Stephen and Shenoy, Krishna and Ganguli, Surya},
	title = {A theory of multineuronal dimensionality, dynamics and measurement},
	elocation-id = {214262},
	year = {2017},
	doi = {10.1101/214262},
	publisher = {Cold Spring Harbor Laboratory},
	URL = {https://www.biorxiv.org/content/early/2017/11/12/214262},
	eprint = {https://www.biorxiv.org/content/early/2017/11/12/214262.full.pdf},
	journal = {bioRxiv}
}

@misc{bak2021secondinternationalverificationneural,
      title={The Second International Verification of Neural Networks Competition (VNN-COMP 2021): Summary and Results}, 
      author={Stanley Bak and Changliu Liu and Taylor Johnson},
      year={2021},
      eprint={2109.00498},
      archivePrefix={arXiv},
      primaryClass={cs.LO},
      url={https://arxiv.org/abs/2109.00498}, 
}

@inproceedings{Krizhevsky2009LearningML,
  title={Learning Multiple Layers of Features from Tiny Images},
  author={Alex Krizhevsky},
  year={2009},
  url={https://api.semanticscholar.org/CorpusID:18268744}
}

@ARTICLE{726791,
  author={Lecun, Y. and Bottou, L. and Bengio, Y. and Haffner, P.},
  journal={Proceedings of the IEEE}, 
  title={Gradient-based learning applied to document recognition}, 
  year={1998},
  volume={86},
  number={11},
  pages={2278-2324},
  doi={10.1109/5.726791}}

@article{alexandr2026robustness,
  title={Robustness Verification of Polynomial Neural Networks},
  author={Alexandr, Yulia and Duan, Hao and Mont{\'u}far, Guido},
  journal={arXiv preprint arXiv:2602.06105},
  year={2026}
}

@article{DBLP:journals/corr/abs-2003-03021,
  author       = {Kai Jia and
                  Martin C. Rinard},
  title        = {Exploiting Verified Neural Networks via Floating Point Numerical Error},
  journal      = {CoRR},
  volume       = {abs/2003.03021},
  year         = {2020},
  url          = {https://arxiv.org/abs/2003.03021},
  eprinttype   = {arXiv},
  eprint       = {2003.03021},
  bibsource    = {dblp computer science bibliography, https://dblp.org}
}

@inproceedings{
zombori2021fooling,
title={Fooling a Complete Neural Network Verifier},
author={D{\'a}niel Zombori and Bal{\'a}zs B{\'a}nhelyi and Tibor Csendes and Istv{\'a}n Megyeri and M{\'a}rk Jelasity},
booktitle={International Conference on Learning Representations},
year={2021},
url={https://openreview.net/forum?id=4IwieFS44l}
}

@misc{madry2019deeplearningmodelsresistant,
      title={Towards Deep Learning Models Resistant to Adversarial Attacks}, 
      author={Aleksander Madry and Aleksandar Makelov and Ludwig Schmidt and Dimitris Tsipras and Adrian Vladu},
      year={2019},
      eprint={1706.06083},
      archivePrefix={arXiv},
      primaryClass={stat.ML},
      url={https://arxiv.org/abs/1706.06083}, 
}

@inproceedings{GDVB,
author = {Xu, Dong and Shriver, David and Dwyer, Matthew B. and Elbaum, Sebastian},
title = {Systematic Generation of Diverse Benchmarks for DNN Verification},
year = {2020},
isbn = {978-3-030-53287-1},
publisher = {Springer-Verlag},
address = {Berlin, Heidelberg},
url = {https://doi.org/10.1007/978-3-030-53288-8_5},
doi = {10.1007/978-3-030-53288-8_5},
booktitle = {Computer Aided Verification: 32nd International Conference, CAV 2020, Los Angeles, CA, USA, July 21–24, 2020, Proceedings, Part I},
pages = {97–121},
numpages = {25},
location = {Los Angeles, CA, USA}
}

@INPROCEEDINGS{relusplitter,
  author={Li, Linhan and Nguyen, ThanhVu},
  booktitle={2025 40th IEEE/ACM International Conference on Automated Software Engineering (ASE)}, 
  title={Destabilizing Neurons to Generate Challenging Neural Network Verification Benchmarks}, 
  year={2025},
  volume={},
  number={},
  pages={1351-1363},
  doi={10.1109/ASE63991.2025.00115}}

\appendix 
\newpage

\section{Robust Constructor Details}
\label{app:constructor_details}

This appendix provides the implementation details for VeriStress-GT's robust constructors described in Section~\ref{sec:robust_constructions}.

\subsection{Exact-Radius via MILP}
\label{app:milp}
Let an $L$-layer ReLU MLP have preactivations $s^{(\ell)}=W^{(\ell)}z^{(\ell-1)}+b^{(\ell)}$ and activations $z^{(\ell)}=\sigma(s^{(\ell)})$, with $z^{(0)}=x_0+\delta$ and $\sigma(t)=\max\{t,0\}$. Given valid preactivation bounds $l_j^{(\ell)}\leq s_j^{(\ell)}\leq u_j^{(\ell)}$ obtained via simple interval arithmetic or Interval Bound Propagation \cite{DBLP:journals/corr/abs-1810-12715}, each ReLU node is encoded using a binary variable $a_j^{(\ell)}\in\{0,1\}$ and standard big-$M$ constraints. The interval bounds are used only to define valid big-$M$ constants and do not relax the network semantics. With binary decision variables, the constraints encode the ReLU graph exactly. For a fixed target class $k\neq y$, the minimum $\ell_\infty$ perturbation required to make class $k$ match or exceed the true-class logit is obtained by solving the following program:

\begin{align*}
\min_{\delta,\, t,\, s,\, z,\, a} \quad & t 
&& \text{ (smallest adversarial radius)}\\
\text{s.t.} \quad 
& -t \le \delta_i \le t, \quad \forall i 
&& \text{ ($\ell_\infty$ ball)} \\[3pt]
& z^{(0)} = x_0 + \delta 
&& \text{ (perturbed input)} \\[3pt]
& s^{(\ell)} = W^{(\ell)} z^{(\ell-1)} + b^{(\ell)}, \ \forall \ell 
&& \text{ (affine layers)} \\[3pt]
& z^{(\ell)}_j \ge 0, \quad 
  z^{(\ell)}_j \ge s^{(\ell)}_j 
&& \text{ (ReLU lower bounds)} \\[3pt]
& z^{(\ell)}_j \le u^{(\ell)}_j a^{(\ell)}_j 
&& \text{ (inactive case)} \\[3pt]
& z^{(\ell)}_j \le s^{(\ell)}_j - l^{(\ell)}_j(1-a^{(\ell)}_j)
&& \text{ (active case)} \\[3pt]
& a^{(\ell)}_j \in \{0,1\}, \ \forall \ell,j 
&& \text{ (ReLU on/off)} \\[3pt]
& f_k(x_0+\delta) \ge f_y(x_0+\delta) 
&& \text{ (misclassification)}.
\end{align*}
The certified radius is $r^\star=\min\limits_{k\neq y}t_k^\star$. Thus, for any chosen $\epsilon<r^\star$, the instance is robust by exact optimization.

\subsection{MEAP Constructor Details}
\label{app:meap}

We give the explicit parameterization used in the mutually exclusive activation pattern constructor. For each pair $p\in[P]$, choose a direction $w_p\in\mathbb R^d$ and a margin parameter $\gamma_p>0$. Define
\begin{equation}
    b_{p,1}=\gamma_p-w_p^\top x_0,
    \qquad
    b_{p,2}=\gamma_p+w_p^\top x_0.
    \label{eq:meap_bias_parameterization}
\end{equation}
Then, the paired preactivations satisfy
\begin{equation}
    z_{p,1}(x_0)=z_{p,2}(x_0)=\gamma_p,
    \qquad
    b_{p,1}+b_{p,2}=2\gamma_p>0.
    \label{eq:meap_nominal_values}
\end{equation}
Thus both neurons are active at the nominal point, but the pair is constructed so that they cannot both be inactive at any input. To see this, observe that for any $x$,
\begin{equation}
    z_{p,1}(x)+z_{p,2}(x)=b_{p,1}+b_{p,2}=2\gamma_p.
    \label{eq:meap_sum_identity}
\end{equation}
Therefore, the two preactivations cannot both be negative. Moreover,
\begin{equation}
\begin{aligned}
    r_p(x)&=\max\{\sigma(z_{p,1}(x)),\sigma(z_{p,2}(x))\} \\
    &=\max\{0,z_{p,1}(x),z_{p,2}(x)\} \\
    &\geq\frac{z_{p,1}(x)+z_{p,2}(x)}{2}=\gamma_p.
\end{aligned}
\label{eq:meap_pair_lower_bound}
\end{equation}
Hence each pair contributes a global lower bound $\gamma_p$. Letting $\gamma=\min\limits_{p\in[P]}\gamma_p$, we have $r_p(x)\geq\gamma$ for all $p$ and all $x$. Since the target logit is $f_y(x)=\min\limits_{p\in[P]}r_p(x)$ and all competing logits are fixed to zero, it follows that
\begin{equation}
    f_y(x)-f_k(x)\geq \gamma
    \qquad
    \forall x,\quad \forall k\neq y.
    \label{eq:meap_global_certificate}
\end{equation}
Thus the MEAP network is robust with margin at least $\gamma$ on $\mathcal B_\epsilon(x_0)$.

The construction is made difficult by choosing parameters so that both ReLU neurons in each pair are unstable over the perturbation box. For the $\ell_\infty$ ball $\mathcal B_\epsilon(x_0)$, the minimum values of the two affine preactivations satisfy
\begin{equation}
    \min_{x\in\mathcal B_\epsilon(x_0)} z_{p,1}(x)=\gamma_p-\epsilon\|w_p\|_1,
    \qquad
    \min_{x\in\mathcal B_\epsilon(x_0)} z_{p,2}(x)= \gamma_p-\epsilon\|w_p\|_1.
    \label{eq:meap_min_values}
\end{equation}
Similarly, both preactivations have maximum value $\gamma_p+\epsilon\|w_p\|_1$ over the box. Therefore, if
\begin{equation}
    0<\gamma_p<\epsilon\|w_p\|_1,
    \label{eq:meap_instability_condition}
\end{equation}
then the interval bounds for both $z_{p,1}$ and $z_{p,2}$ cross zero. Each ReLU is individually unstable, even though the pair satisfies the exact lower bound \eqref{eq:meap_pair_lower_bound}. This creates the intended verifier stress, as relaxations that treat the two ReLUs independently may admit the infeasible relaxed state $\sigma(z_{p,1})=\sigma(z_{p,2})=0$, while the true network prevents it. Finally, the aggregation operations can be implemented using ReLU identities. For scalars $a,b$, note:
\begin{equation}
    \max\{a,b\}=b+\sigma(a-b), \qquad
    \min\{a,b\}=a-\sigma(a-b).
    \label{eq:relu_minmax_identities}
\end{equation}
Thus the pairwise maxima $r_p(x)$ and the final minimum over pairs can be represented as a standard piecewise-linear ReLU network.

\subsection{Input-Corner Stress Constructor Details}
\label{app:corner}

We give the formal certificate for the input-corner construction. Let $A\subseteq[d]$ be an active coordinate set with $|A|=d_{act}$, and let $\mathcal Q_\epsilon(x_0;A)$ denote the projection of $\mathcal B_\epsilon(x_0)$ onto these active coordinates. Its vertices are:
\begin{equation}
    \mathcal V_\epsilon(x_0;A)=\{(x_0)_A+\epsilon s:s_i\in\{-1,+1\} \text{ for } i\in A, s_i=0 \text{ for } i \notin A\}.
    \label{eq:corner_vertices}
\end{equation}
The constructed logits depend only on $x_A$. We set $f_y(x)=0$ and $f_k(x)=h_k(x_A)-\beta_k$ for $k\neq y$, where each $h_k:\mathbb R^{d_{act}}\to\mathbb R$ is convex on $\mathcal Q_\epsilon(x_0;A)$. For a desired margin $\gamma>0$, define
\begin{equation}
    \beta_k=\max_{v\in\mathcal V_\epsilon(x_0;A)}h_k(v)+\gamma.
    \label{eq:app_corner_beta}
\end{equation}

\begin{proposition}[Input-corner certificate]
If each $h_k$ is convex on $\mathcal Q_\epsilon(x_0;A)$ and $\beta_k$ is chosen as in \eqref{eq:app_corner_beta}, then the constructed classifier is robust on $\mathcal B_\epsilon(x_0)$ with margin at least $\gamma$.
\end{proposition}

\begin{proof}
For any $x\in\mathcal B_\epsilon(x_0)$, we have $x_A\in\mathcal Q_\epsilon(x_0;A)$. Since $\mathcal Q_\epsilon(x_0;A)$ is a box, $x_A$ can be written as a convex combination of its vertices:

\begin{equation}
    x_A = \sum_{v\in\mathcal V_\epsilon(x_0;A)} \lambda_v v_A,
    \qquad \lambda_v\geq 0,
    \qquad \sum_{v\in\mathcal V_\epsilon(x_0;A)} \lambda_v=1
\end{equation}

And by the convexity of $h_k$. we have:

\begin{equation}
h_k(x_A)\leq
    \sum_{v\in\mathcal V_\epsilon(x_0;A)} \lambda_v h_k(v_A) \leq
    \max_{v\in\mathcal V_\epsilon(x_0;A)} h_k(v_A).
    \label{eq:corner_convex_max}
\end{equation}
Therefore, for any $x\in\mathcal B_\epsilon(x_0)$ and $k\neq y$:
\begin{equation}
    \mu_k(x) =f_y(x)-f_k(x)=\beta_k-h_k(x_A)
    \geq
    \beta_k-\max_{v\in\mathcal V_\epsilon(x_0;A)}h_k(v)=\gamma.
    \label{eq:corner_margin}
\end{equation}
Since this holds for every $k\neq y$, we have $\mu(x)=\min\limits_{k\neq y}\mu_k(x)\geq\gamma$ throughout $\mathcal B_\epsilon(x_0)$.
\end{proof}

A ReLU implementation is obtained by choosing
\begin{equation}
    h_k(u)=\sum_{j=1}^{J} c_{k,j}\sigma(a_j^\top u+b_j),
    \qquad
    c_{k,j}\geq 0
    \label{eq:corner_relu_hinges}
\end{equation}
Because $u\mapsto\sigma(a_j^\top u+b_j)$ is convex and nonnegative linear combinations preserve convexity, each $h_k$ is convex. To make the instance difficult, the constructor may place hinges near the nominal input by setting, for example, $b_j=-a_j^\top (x_0)_A$. Then the $j$-th hinge preactivation is zero at $x_0$, and its interval over the active-coordinate box satisfies the following:
\begin{equation}
    a_j^\top(u-(x_0)_A)
    \in
    [-\epsilon\|a_j\|_1,\epsilon\|a_j\|_1],
    \label{eq:corner_hinge_interval}
\end{equation}
so the corresponding ReLU is unstable under IBP bounds. Increasing the number of hinges $J$, the active dimension $d_{\mathrm{act}}$, or the hinge scales $\|a_j\|_1$ increases the number and magnitude of unstable hinge contributions while preserving the robustness certificate.

\subsection{Contractive CNN Constructor Details}
\label{app:contractive_cnn}

We give the formal certificate for the deep contractive CNN constructor. Write the convolutional feature extractor as $\Phi = C_D\circ C_{D-1}\circ\cdots\circ C_1\circ P$ where $P$ is an optional front-end map and $C_1,\dots,C_D$ are contractive convolutional blocks. Let $L_{\mathrm{front}}$ denote the induced $\ell_\infty$ Lipschitz constant of $P$, with $L_{\mathrm{front}}=1$ if no front-end map is used. Assume each block $C_\ell$ is $\ell_\infty$-Lipschitz with constant at most $\lambda<1$. Then, for all $x,x'$,
\begin{equation}
    \|\Phi(x)-\Phi(x')\|_\infty
    \leq
    L_{\mathrm{front}}\lambda^D\|x-x'\|_\infty.
    \label{eq:app_contract_lip}
\end{equation}
This follows directly by composing the Lipschitz constants of $P,C_1,\dots,C_D$. 

Next, define the centered target logit by $f_y(x) = \Gamma+w_y^\top(\Phi(x)-\Phi(x_0))$ and set $f_k(x) =0 \; \forall k\neq y$. The following Proposition defines our robustness claim for the constructor:

\begin{proposition}[Contractive CNN certificate]
If
\begin{equation}
    \Gamma>\|w_y\|_1L_{\mathrm{front}}\lambda^D\epsilon,
\end{equation}
then the classifier $\Phi$ is robust on $\mathcal B_\epsilon(x_0)$.
\end{proposition}

\begin{proof}
For any $x=x_0+\delta$ with $\|\delta\|_\infty\leq\epsilon$, H\"older's inequality and \eqref{eq:app_contract_lip} give
\begin{equation}
\begin{aligned}
    \left|w_y^\top(\Phi(x)-\Phi(x_0))\right|
    &\leq \|w_y\|_1\|\Phi(x)-\Phi(x_0)\|_\infty \\
    &\leq \|w_y\|_1L_{\mathrm{front}}\lambda^D\|x-x_0\|_\infty \\
    &\leq \|w_y\|_1L_{\mathrm{front}}\lambda^D\epsilon.
\end{aligned}
\label{eq:app_contract_drift}
\end{equation}
Therefore,
\begin{equation}
    f_y(x)
    \geq
    \Gamma-\|w_y\|_1L_{\mathrm{front}}\lambda^D\epsilon
    > 0
    = f_k(x) \qquad \forall k\neq y.
\end{equation}
Thus, $\mu(x)=\min_{k\neq y}(f_y(x)-f_k(x))>0$ throughout $\mathcal B_\epsilon(x_0)$, so the instance is robust.
\end{proof}

The centered logit is convenient because the nominal target logit is exactly $\Gamma$ and all competing logits are fixed to zero. An un-centered linear logit head can also be used. Suppose
\begin{equation}
    f_i(x)=b_i+w_i^\top\Phi(x),
    \qquad i\in[c].
\end{equation}
Then for each $k\neq y$, define the nominal class-wise margin $\mu_k(x_0)=f_y(x_0)-f_k(x_0)$. Then, $\mu_k(x)=\mu_k(x_0)+(w_y-w_k)^\top(\Phi(x)-\Phi(x_0))$.

Therefore, by the same Lipschitz argument:
\begin{equation}
    \mu_k(x)
    \geq
    \mu_k(x_0)
    -
    \|w_y-w_k\|_1L_{\mathrm{front}}\lambda^D\epsilon.
\end{equation}
Thus, the uncentered head is robust whenever
\begin{equation}
    \mu_k(x_0)
    >
    \|w_y-w_k\|_1L_{\mathrm{front}}\lambda^D\epsilon
    \qquad
    \text{for all } k\neq y.
    \label{eq:app_contract_uncentered}
\end{equation}

\subsection{Paired-Bias CNN Constructor Details}
\label{app:paired_bias_cnn}

We give the formal certificate and instability parameterization for the paired-bias CNN constructor. Let $Z=\Psi(x)$ denote the feature map produced by any upstream convolutional layers. For each pair $i\in[P]$, the two channels share the same convolutional filter $\mathcal W_i$ and differ only in their biases. Write $ s^{(i)}=\mathcal W_i\star Z\in\mathbb R^{H_{\mathrm{sp}}\times W_{\mathrm{sp}}}$ so that at spatial location $(h,w)$, the paired activations are $\sigma(s_{h,w}^{(i)}+b_i)$ and $\sigma(s_{h,w}^{(i)}+c_i)$, where $b_i>c_i$. The certificate follows from monotonicity of the ReLU function. For any $s\in\mathbb R$ and any $b_i>c_i$, we have $s+b_i>s+c_i$, and since $\sigma$ is monotone and nondecreasing, then:
\begin{equation}
    \sigma(s+b_i)\geq \sigma(s+c_i).
    \label{eq:app_paired_bias_monotone}
\end{equation}
Therefore, each paired difference is nonnegative:
\begin{equation}
    \sigma(s+b_i)-\sigma(s+c_i)\geq 0
    \qquad
    \forall s\in\mathbb R.
    \label{eq:app_paired_bias_nonnegative}
\end{equation}

Next, define the logit outputs as the following:
\begin{equation}
    f_y(x)
    =
    \Gamma+
    \frac{1}{P H_{\mathrm{sp}} W_{\mathrm{sp}}}
    \sum_{i=1}^{P}
    \sum_{h=1}^{H_{\mathrm{sp}}}
    \sum_{w=1}^{W_{\mathrm{sp}}}
    \left[
        \sigma(s_{h,w}^{(i)}+b_i)-\sigma(s_{h,w}^{(i)}+c_i)
    \right],
    \qquad
    f_k(x)=0\quad \text{for all }k\neq y.
    \label{eq:app_paired_bias_logit}
\end{equation}

\begin{proposition}[Paired-bias CNN certificate]
If $\Gamma>0$ and $b_i>c_i$ for every pair $i\in[P]$, then the classifier defined in \eqref{eq:app_paired_bias_logit} is robust with margin at least $\Gamma$ for every input $x$.
\end{proposition}

\begin{proof}
By \eqref{eq:app_paired_bias_nonnegative}, every summand in \eqref{eq:app_paired_bias_logit} is nonnegative. Hence $f_y(x)\geq \Gamma$ for every input $x$. Since $f_k(x)=0$ for all $k\neq y$, we have
\[
    f_y(x)-f_k(x)\geq \Gamma>0
\]
for every $x$ and every $k\neq y$. Therefore
\[
    \mu(x)=\min_{k\neq y}(f_y(x)-f_k(x))\geq \Gamma,
\]
so the classifier is robust on any perturbation set $\mathcal B_\epsilon(x_0)$.
\end{proof}

To make the instance difficult for relaxation-based verifiers, the constructor chooses biases so that many paired ReLUs are unstable over the perturbation set. Suppose interval propagation gives bounds
\[
    s_{h,w}^{(i)}\in
    [s_{\mathrm{lo},h,w}^{(i)},s_{\mathrm{hi},h,w}^{(i)}]
\]
for the shared convolutional response over $\mathcal B_\epsilon(x_0)$. In practice, we choose a channelwise center $t_i$, such as a spatial average of interval midpoints, and a bias half-gap $\Delta_i>0$, and set
\begin{equation}
    b_i=-t_i+\Delta_i,
    \qquad
    c_i=-t_i-\Delta_i.
    \label{eq:app_paired_bias_biases}
\end{equation}
Then $b_i-c_i=2\Delta_i>0$, so the monotonicity certificate remains valid. For a given spatial location $(h,w)$, the shifted intervals are
\begin{equation}
    s_{h,w}^{(i)}+b_i
    \in
    [s_{\mathrm{lo},h,w}^{(i)}-t_i+\Delta_i,\;
     s_{\mathrm{hi},h,w}^{(i)}-t_i+\Delta_i],
\end{equation}
and
\begin{equation}
    s_{h,w}^{(i)}+c_i
    \in
    [s_{\mathrm{lo},h,w}^{(i)}-t_i-\Delta_i,\;
     s_{\mathrm{hi},h,w}^{(i)}-t_i-\Delta_i].
    \label{eq:app_paired_bias_intervals}
\end{equation}
When these intervals cross zero, the corresponding ReLUs are unstable under interval bounds. Therefore, the constructor can create many unstable paired activations while preserving the exact nonnegativity guarantee. Independent relaxations may lose the fact that the two ReLUs are functions of the same scalar $s_{h,w}^{(i)}$, and can therefore produce a negative lower bound on a quantity that is globally nonnegative.

\subsection{Fixed-Ordering Softmax Attention Constructor Details}
\label{app:fixed_attention}

The robustness certificate for the fixed-ordering attention constructor has two parts. First, we show that the nominal row-wise score maximizers remain unchanged throughout the perturbation set. This ensures that the attention score pattern is stable. Second, we bound the total change in the attention module (i.e. score matrix times the value matrix $XW_V$) output. This bound, combined with a Lipschitz downstream head and a sufficiently large nominal classification margin, gives the final robustness certificate. Thus, the first proposition certifies score-pattern stability, while the second proposition certifies classification robustness. Let $X=X_0+\Delta X$, and write $\delta^{(i)}$ for the $i$-th row of $\Delta X$. We assume $\|\Delta X\|_\infty\leq \epsilon$, so $\|\delta^{(i)}\|_2\leq \epsilon\sqrt{d_{\mathrm{tok}}}$ for every $i$.

\paragraph{Part 1: Score-gap stability.}
For any $j\neq\pi_i^\star$, since $X=X_0+\Delta X$, we have
\begin{align}
    S_{i,\pi_i^\star}(X)-S_{ij}(X)
    &=
    (x_0^{(i)}+\delta^{(i)})^\top M
    \left[
        (x_0^{(\pi_i^\star)}+\delta^{(\pi_i^\star)})
        -
        (x_0^{(j)}+\delta^{(j)})
    \right] \notag\\
    &=
    {x_0^{(i)}}^\top M
    \left(x_0^{(\pi_i^\star)}-x_0^{(j)}\right)
    +
    {\delta^{(i)}}^\top M
    \left(x_0^{(\pi_i^\star)}-x_0^{(j)}\right) \notag\\
    &\qquad
    +
    {x_0^{(i)}}^\top M
    \left(\delta^{(\pi_i^\star)}-\delta^{(j)}\right)
    +
    {\delta^{(i)}}^\top M
    \left(\delta^{(\pi_i^\star)}-\delta^{(j)}\right).
    \label{eq:app_attention_gap_expand_full}
\end{align}
The first term is the nominal score gap,
\[
    {x_0^{(i)}}^\top M
    \left(x_0^{(\pi_i^\star)}-x_0^{(j)}\right)
    =
    S_{i,\pi_i^\star}(X_0)-S_{ij}(X_0)
    =
    \Delta_{ij}.
\]
Therefore,
\begin{align}
    &\left[S_{i,\pi_i^\star}(X)-S_{ij}(X)\right]-\Delta_{ij} \notag\\
    &\qquad =
    {\delta^{(i)}}^\top M
    \left(x_0^{(\pi_i^\star)}-x_0^{(j)}\right)
    +
    {x_0^{(i)}}^\top M
    \left(\delta^{(\pi_i^\star)}-\delta^{(j)}\right)+
    {\delta^{(i)}}^\top M
    \left(\delta^{(\pi_i^\star)}-\delta^{(j)}\right).
    \label{eq:app_attention_gap_expansion}
\end{align}
Using $|u^\top Mv|\leq \|M\|_{\mathrm{op}}\|u\|_2\|v\|_2$, $\|\delta^{(i)}\|_2\leq\epsilon\sqrt{d_{\mathrm{tok}}}$, and $\|\delta^{(\pi_i^\star)}-\delta^{(j)}\|_2\leq 2\epsilon\sqrt{d_{\mathrm{tok}}}$, we obtain
\begin{equation}
    \left|
    \left[S_{i,\pi_i^\star}(X)-S_{ij}(X)\right]-\Delta_{ij}
    \right|
    \leq
    \epsilon C_{ij}(\epsilon),
    \label{eq:app_attention_gap_bound}
\end{equation}
where
\begin{equation}
    C_{ij}(\epsilon)
    =
    \sqrt{d_{\mathrm{tok}}}\,\|M\|_{\mathrm{op}}
    \left(
        \|x_0^{(\pi_i^\star)}-x_0^{(j)}\|_2
        +
        2\|x_0^{(i)}\|_2
    \right)
    +
    2\epsilon d_{\mathrm{tok}}\,\|M\|_{\mathrm{op}}.
    \label{eq:app_attention_Cij}
\end{equation}

\begin{proposition}[Attention score-pattern stability]
If $\Delta_{ij}>\epsilon C_{ij}(\epsilon)$ for every $i$ and every $j\neq\pi_i^\star$, then $\pi(X)=\pi^\star$ for every $X$ satisfying $\|X-X_0\|_\infty\leq\epsilon$.
\end{proposition}

\begin{proof}
By \eqref{eq:app_attention_gap_bound},
\begin{equation}
    S_{i,\pi_i^\star}(X)-S_{ij}(X)
    \geq
    \Delta_{ij}-\epsilon C_{ij}(\epsilon)
    >
    0
\end{equation}
for every $j\neq\pi_i^\star$. Hence the row-wise maximizer remains $\pi_i^\star$ for every row $i$.
\end{proof}

\paragraph{Part 2: Attention-output perturbation.}
We next bound the attention output. Let $A_0=A(X_0)$, $\widetilde A=A(X)$, $V_0=X_0W_V$, and $\Delta V=\Delta XW_V$. The exact decomposition is
\begin{equation}
    \mathrm{Attn}(X)-\mathrm{Attn}(X_0)
    =
    (\widetilde A-A_0)V_0
    +
    A_0\Delta V
    +
    (\widetilde A-A_0)\Delta V.
    \label{eq:app_attention_decomposition}
\end{equation}

For every score entry, bilinearity gives:
\begin{equation}
    |S_{ij}(X)-S_{ij}(X_0)|
    \leq
    \epsilon\sqrt{d_{\mathrm{tok}}}\,\|M\|_{\mathrm{op}}
    \left(
        \|x_0^{(i)}\|_2+\|x_0^{(j)}\|_2
    \right)
    +
    \epsilon^2 d_{\mathrm{tok}}\,\|M\|_{\mathrm{op}}.
    \label{eq:app_attention_score_entry}
\end{equation}
Define
\begin{equation}
    \bar B_S(\epsilon)
    =
    \max_{i,j}
    \left[
        \sqrt{d_{\mathrm{tok}}}\,\|M\|_{\mathrm{op}}
        \left(
            \|x_0^{(i)}\|_2+\|x_0^{(j)}\|_2
        \right)
        +
        \epsilon d_{\mathrm{tok}}\,\|M\|_{\mathrm{op}}
    \right].
    \label{eq:app_attention_BS}
\end{equation}
Then $\|S(X)-S(X_0)\|_\infty\leq\epsilon\bar B_S(\epsilon)$.

Using the softmax Jacobian bound $\|\nabla \mathrm{softmax}(z)\|_{\mathrm{op}}\leq \frac{1}{2}$, for each attention row $i$ we have
\begin{equation}
    \|\widetilde a_i-a_i^0\|_1
    \leq
    \frac{n}{2}\epsilon\bar B_S(\epsilon).
    \label{eq:app_attention_weight_bound}
\end{equation}
Also, $\|\Delta V\|_{2,\infty}\leq \epsilon\sqrt{d_{\mathrm{tok}}}\,\|W_V\|_{\mathrm{op}}$, where $\|Z\|_{2,\infty}=\max_i\|Z[i,:]\|_2$. Applying these bounds to the three terms in \eqref{eq:app_attention_decomposition} results in:
\begin{equation}
    \|\mathrm{Attn}(X)-\mathrm{Attn}(X_0)\|_{2,\infty}
    \leq
    \epsilon L_{\mathrm{attn}}(\epsilon),
    \label{eq:app_attention_output_2inf}
\end{equation}
where
\begin{equation}
    L_{\mathrm{attn}}(\epsilon)
    =
    \frac{n}{2}\bar B_S(\epsilon)\|V_0\|_{2,\infty}
    +
    \sqrt{d_{\mathrm{tok}}}\,\|W_V\|_{\mathrm{op}}
    +
    \frac{n}{2}\epsilon\bar B_S(\epsilon)\sqrt{d_{\mathrm{tok}}}\,\|W_V\|_{\mathrm{op}}.
    \label{eq:app_attention_Lattn}
\end{equation}
Since $\|Z\|_F\leq\sqrt n\|Z\|_{2,\infty}$, this implies
\begin{equation}
    \|\mathrm{Attn}(X)-\mathrm{Attn}(X_0)\|_F
    \leq
    \sqrt n L_{\mathrm{attn}}(\epsilon)\epsilon.
    \label{eq:app_attention_output_fro}
\end{equation}

\begin{proposition}[Fixed-pattern attention robustness]
Let $f=h\circ\mathrm{Attn}$, and suppose each logit of $h$ is $L_h$-Lipschitz with respect to the Frobenius norm. If the score-pattern stability condition holds and $\mu(X_0)>2L_h\sqrt n L_{\mathrm{attn}}(\epsilon)\epsilon$, then $f$ is robust on $\mathcal B_\epsilon(X_0)$.
\end{proposition}

\begin{proof}
By \eqref{eq:app_attention_output_fro} and the logit-wise $L_h$-Lipschitzness of $h$, for each class $r$,
\begin{equation}
    |f_r(X)-f_r(X_0)|
    \leq
    L_h\sqrt n L_{\mathrm{attn}}(\epsilon)\epsilon.
\end{equation}
Therefore, for any $k\neq y$,
\begin{equation}
\begin{aligned}
    f_y(X)-f_k(X)
    &\geq
    f_y(X_0)-f_k(X_0)
    -
    |f_y(X)-f_y(X_0)|
    -
    |f_k(X)-f_k(X_0)| \\
    &\geq
    \mu(X_0)-2L_h\sqrt n L_{\mathrm{attn}}(\epsilon)\epsilon
    >
    0.
\end{aligned}
\end{equation}
Thus the predicted class remains $y$ for every $X\in\mathcal B_\epsilon(X_0)$.
\end{proof}

\subsection{Dominant-Key Linear Attention Constructor Details}
\label{app:linear_dominance_attention}

The robustness certificate has three steps, and we use the same notation as in the main text. First, we show that the dominance condition implies that each attention row is close to the value vector associated with its dominant key. Second, we show that this closeness gives a uniform bound on the full attention-output drift. Third, this output-drift bound combines with the nominal classification margin to certify robustness.

\paragraph{Part 1: Closeness to dominant value.}
The dominance condition says that for query row $i$, the unnormalized weight assigned to key $j_i^\star$ is at least $\rho_i$ times the total weight assigned to all other keys. The next lemma makes the consequence precise:

\begin{lemma}[Closeness to the dominant value]
\label{lem:linear_dominant_closeness}
If row $i$ satisfies the dominance condition in \eqref{eq:linear_attention_dominance} throughout $\mathcal B_\epsilon(X_0)$, then for every $X\in\mathcal B_\epsilon(X_0)$,
\begin{equation}
    \left\|
    \mathrm{Attn}_{\mathrm{lin}}(X)[i,:]
    -
    V(X)[j_i^\star,:]
    \right\|_2
    \leq
    \frac{1}{1+\rho_i}
    \max_{j\neq j_i^\star}
    \left\|
    V(X)[j,:]-V(X)[j_i^\star,:]
    \right\|_2.
    \label{eq:app_linear_closeness}
\end{equation}
\end{lemma}

\begin{proof}
Let $S_i(X)=\sum\limits_{j\neq j_i^\star}w_{ij}(X)$ and $w_i^\star(X)=w_{i j_i^\star}(X)$. We also have that $w_i^\star(X)\geq \rho_i S_i(X)$ by \eqref{eq:linear_attention_dominance}. Therefore:
\begin{equation}
    \alpha_{i j_i^\star}(X)
    =
    \frac{w_i^\star(X)}{w_i^\star(X)+S_i(X)}
    \geq
    \frac{\rho_i}{1+\rho_i},
    \qquad
    1-\alpha_{i j_i^\star}(X)
    \leq
    \frac{1}{1+\rho_i}.
    \label{eq:app_linear_mass_bound}
\end{equation}
Using the definition of linear attention,
\begin{equation}
    \mathrm{Attn}_{\mathrm{lin}}(X)[i,:]-V(X)[j_i^\star,:]
    =
    \sum_{j\neq j_i^\star}
    \alpha_{ij}(X)
    \left(
        V(X)[j,:]-V(X)[j_i^\star,:]
    \right).
    \label{eq:app_linear_deviation_identity}
\end{equation}
Applying the triangle inequality and \eqref{eq:app_linear_mass_bound} gives
\begin{equation}
\begin{aligned}
    \left\|
    \mathrm{Attn}_{\mathrm{lin}}(X)[i,:]-V(X)[j_i^\star,:]
    \right\|_2
    &\leq
    \sum_{j\neq j_i^\star}
    \alpha_{ij}(X)
    \left\|
    V(X)[j,:]-V(X)[j_i^\star,:]
    \right\|_2 \\
    &\leq
    \left(1-\alpha_{i j_i^\star}(X)\right)
    \max_{j\neq j_i^\star}
    \left\|
    V(X)[j,:]-V(X)[j_i^\star,:]
    \right\|_2 \\
    &\leq
    \frac{1}{1+\rho_i}
    \max_{j\neq j_i^\star}
    \left\|
    V(X)[j,:]-V(X)[j_i^\star,:]
    \right\|_2.
\end{aligned}
\end{equation}
\end{proof}

\paragraph{Part 2: Attention-output perturbation.}
The previous lemma controls each attention row relative to its dominant value vector. We now compare the full attention outputs at $X$ and $X_0$ by inserting the dominant value vector at both inputs. For compactness, let $V_0=V(X_0)$ and set $L_V=\sqrt{d_{\mathrm{tok}}}\|W_V\|_{\mathrm{op}}$. If $X=X_0+\Delta X$ and $\|\Delta X\|_\infty\leq\epsilon$, then each token perturbation satisfies $\|\delta^{(i)}\|_2\leq\epsilon\sqrt{d_{\mathrm{tok}}}$, so
\begin{equation}
    \|V(X)[j,:]-V_0[j,:]\|_2
    =
    \|\delta^{(j)}W_V\|_2
    \leq
    \epsilon L_V.
    \label{eq:app_value_drift}
\end{equation}

\begin{proposition}[Dominant-key attention-output perturbation]
\label{prop:linear_attention_perturbation}
Suppose every row satisfies the dominance condition in \eqref{eq:linear_attention_dominance} throughout $\mathcal B_\epsilon(X_0)$, and let $\rho=\min\limits_i\rho_i$. Then for every $X\in\mathcal B_\epsilon(X_0)$,
\begin{equation}
    \|\mathrm{Attn}_{\mathrm{lin}}(X)-\mathrm{Attn}_{\mathrm{lin}}(X_0)\|_{2,\infty}
    \leq
    \frac{2}{1+\rho}D_V(X_0)
    +
    \left(
        1+\frac{2}{1+\rho}
    \right)
    \epsilon L_V.
    \label{eq:app_linear_attention_bound}
\end{equation}
\end{proposition}

\begin{proof}
Fix a row $i$. Insert the value vector associated with the dominant key at both $X$ and $X_0$:
\begin{equation}
\begin{aligned}
    &\left\|
    \mathrm{Attn}_{\mathrm{lin}}(X)[i,:]
    -
    \mathrm{Attn}_{\mathrm{lin}}(X_0)[i,:]
    \right\|_2 \\
    &\quad\leq
    \left\|
    \mathrm{Attn}_{\mathrm{lin}}(X)[i,:]-V(X)[j_i^\star,:]
    \right\|_2
    +
    \left\|
    V(X)[j_i^\star,:]-V_0[j_i^\star,:]
    \right\|_2 \\
    &\qquad+
    \left\|
    V_0[j_i^\star,:]-\mathrm{Attn}_{\mathrm{lin}}(X_0)[i,:]
    \right\|_2.
\end{aligned}
\label{eq:app_linear_three_terms}
\end{equation}
By Lemma~\ref{lem:linear_dominant_closeness} applied at $X$,
\begin{equation}
    \left\|
    \mathrm{Attn}_{\mathrm{lin}}(X)[i,:]-V(X)[j_i^\star,:]
    \right\|_2
    \leq
    \frac{1}{1+\rho}
    \max_{j\neq j_i^\star}
    \|V(X)[j,:]-V(X)[j_i^\star,:]\|_2.
    \label{eq:app_linear_term_A}
\end{equation}
For every $j\neq j_i^\star$, adding and subtracting nominal values gives
\begin{equation}
\begin{aligned}
    \|V(X)[j,:]-V(X)[j_i^\star,:]\|_2
    &\leq
    \|V_0[j,:]-V_0[j_i^\star,:]\|_2 \\
    &\quad+
    \|V(X)[j,:]-V_0[j,:]\|_2
    +
    \|V(X)[j_i^\star,:]-V_0[j_i^\star,:]\|_2 \\
    &\leq
    D_V(X_0)
    +
    2\epsilon L_V.
\end{aligned}
\label{eq:app_linear_value_spread_perturbed}
\end{equation}
Thus the first term in \eqref{eq:app_linear_three_terms} is bounded by $\frac{D_V(X_0)+2\epsilon L_V}{(1+\rho)}$. The middle term is bounded by $\epsilon L_V$ using \eqref{eq:app_value_drift}. The last term is bounded by $\frac{D_V(X_0)}{(1+\rho)}$ by Lemma~\ref{lem:linear_dominant_closeness} applied at $X_0$. Summing the three bounds gives
\begin{equation}
    \left\|
    \mathrm{Attn}_{\mathrm{lin}}(X)[i,:]-
    \mathrm{Attn}_{\mathrm{lin}}(X_0)[i,:]
    \right\|_2
    \leq
    \frac{2}{1+\rho}D_V(X_0)
    +
    \left(
        1+\frac{2}{1+\rho}
    \right)
    \epsilon L_V.
\end{equation}
Taking the maximum over rows gives the $\|\cdot\|_{2,\infty}$ bound.
\end{proof}

\paragraph{Part 3: Robustness from output drift.}
The previous proposition bounds the drift of the linear-attention output. The final step is the same argument used in the fixed-ordering softmax constructor proof, i.e. if the downstream head cannot move the logits enough to close the nominal margin, then the predicted class is fixed.

\begin{proposition}[Dominant-key linear attention robustness certificate]
\label{prop:linear_attention_robustness}
Let $f=h\circ\mathrm{Attn}_{\mathrm{lin}}$, and suppose each logit of $h$ is $L_h$-Lipschitz with respect to the Frobenius norm. Define
\begin{equation}
    \Delta_{\mathrm{lin}}
    =
    \frac{2}{1+\rho}D_V(X_0)
    +
    \left(
        1+\frac{2}{1+\rho}
    \right)
    \epsilon L_V.
    \label{eq:app_linear_delta}
\end{equation}
If $\mu(X_0)>2L_h\sqrt n\,\Delta_{\mathrm{lin}}$, then $f$ is robust on $\mathcal B_\epsilon(X_0)$.
\end{proposition}

\begin{proof}
By Proposition~\ref{prop:linear_attention_perturbation}, $\|\mathrm{Attn}_{\mathrm{lin}}(X)-\mathrm{Attn}_{\mathrm{lin}}(X_0)\|_{2,\infty}\leq\Delta_{\mathrm{lin}}$. Therefore,
\begin{equation}
    \|\mathrm{Attn}_{\mathrm{lin}}(X)-\mathrm{Attn}_{\mathrm{lin}}(X_0)\|_F
    \leq
    \sqrt n\,\Delta_{\mathrm{lin}}.
\end{equation}
Since each logit of $h$ is $L_h$-Lipschitz, for every class $r$,
\begin{equation}
    |f_r(X)-f_r(X_0)|
    \leq
    L_h\sqrt n\,\Delta_{\mathrm{lin}}.
\end{equation}
Therefore, for any $k\neq y$,
\begin{equation}
\begin{aligned}
    f_y(X)-f_k(X)
    &\geq
    f_y(X_0)-f_k(X_0)
    -
    |f_y(X)-f_y(X_0)|
    -
    |f_k(X)-f_k(X_0)| \\
    &\geq
    \mu(X_0)-2L_h\sqrt n\,\Delta_{\mathrm{lin}}
    >
    0.
\end{aligned}
\end{equation}
Thus the predicted class remains $y$ for every $X\in\mathcal B_\epsilon(X_0)$.
\end{proof}

\subsection{Polynomial Network Constructor}\label{app:pnn}

\paragraph{Construction and robustness certificate}\label{app:pnn-construction}

We describe the polynomial network constructor in more detail. The construction applies to binary classifiers whose margin is a polynomial function of the input. Let $f:\mathbb R^n\to \mathbb R^2$ be a binary classifier, and write
$\mu(x)=f_1(x)-f_2(x)$ for its margin polynomial. The decision boundary is the real algebraic hypersurface
$D=\{x\in\mathbb R^n:\mu(x)=0\}$. A point $x_0$ is robust at radius $\epsilon$ with respect to a norm $\|\cdot\|$ if $\mu$ has constant sign on $B_\epsilon(x_0)=\{x:\|x-x_0\|\leq \epsilon\}$. Equivalently, $B_\epsilon(x_0)\cap D=\emptyset$.

The constructor proceeds in three steps. First, it samples a smooth point $p\in D$, so that $\nabla \mu(p)\neq 0$. Second, it chooses a norm-adapted normal direction $u(p)$ pointing to one side of the decision boundary. More precisely, $u(p)$ is normalized in the chosen norm and satisfies $\nabla \mu(p)\cdot u(p)>0$. Third, for a target verification radius $\epsilon>0$ and buffer $\delta>0$, it sets $x_0=p+(\epsilon+\delta)u(p)$. The sampled point $p$ is then a known boundary point at distance $\epsilon+\delta$ from $x_0$, and hence lies just outside the perturbation set $B_\epsilon(x_0)$.

The preceding step is only a local construction. It guarantees that one known point of the decision boundary lies outside the verification set, but it does not rule out the possibility that another branch or component of $D$ enters $B_\epsilon(x_0)$. The robustness label is therefore obtained from the following global separation certificate.

\begin{proposition}
\label{prop:pnn-global-certificate}
If $B_\epsilon(x_0)\cap D=\emptyset$, then $x_0$ is robust at radius $\epsilon$.
\end{proposition}

\begin{proof}
Since $B_\epsilon(x_0)$ is connected and $\mu$ is continuous, the image $\mu(B_\epsilon(x_0))$ is connected in $\mathbb R$. If $B_\epsilon(x_0)\cap D=\emptyset$, then $\mu(x)\neq 0$ for every $x\in B_\epsilon(x_0)$. Hence $\mu$ cannot change sign on $B_\epsilon(x_0)$. Therefore the classifier assigns the same label to every point in $B_\epsilon(x_0)$, so $x_0$ is robust at radius~$\epsilon$.
\end{proof}

Thus the constructor separates the task of generating near-boundary candidates from the task of certifying their ground-truth labels. The perturbation step places $x_0$ close to the algebraic decision boundary with a prescribed buffer $\delta$, while the global separation check certifies that the full perturbation set remains on one side of $D$. Decreasing $\delta$ produces instances that are robust but closer to failure, and therefore harder for verifiers to certify.

\paragraph{Boundary sampling and shallow polynomial networks}
\label{app:pnn-sampling}

We now describe how boundary points are sampled and how the construction is instantiated for the shallow polynomial networks used in the experiments. Since $\mu$ is a polynomial, points on $D=\{\mu=0\}$ can be obtained by slicing $D$ with affine linear spaces. In our setting $D$ is a hypersurface, so a generic affine line is enough. We choose a random point $a\in\mathbb R^n$ and a random direction $v\in\mathbb R^n$, and restrict $\mu$ to the line $\ell(t)=a+tv$. This gives the univariate polynomial $\mu_\ell(t)=\mu(a+tv)$. Each real root of $\mu_\ell$ gives a boundary point $p=\ell(t)\in D$. For a generic line, the intersection is transverse away from the singular locus of $D$, so sampled points are smooth unless they also satisfy $\nabla\mu(p)=0$. We discard such points.

In the experiments, we use one-hidden-layer polynomial networks with architecture $(n,h,2)$,
\[
    f(x)=W_2(W_1x+b_1)^d+b_2,
\]
where the power is applied entrywise, $W_1\in\mathbb R^{h\times n}$, $b_1\in\mathbb R^h$, $W_2\in\mathbb R^{2\times h}$, and $b_2\in\mathbb R^2$. Writing $W_{2,1}$ and $W_{2,2}$ for the two rows of $W_2$, the margin polynomial is
\[
    \mu(x)
    =
    (W_{2,1}-W_{2,2})\cdot (W_1x+b_1)^d
    +
    (b_{2,1}-b_{2,2}).
\]
Thus the decision boundary $D=\{\mu=0\}$ is a real algebraic hypersurface of degree at most $d$.

For the $\ell_\infty$ perturbation sets used in our experiments, we take the norm-adapted normal direction to be
$u(p)=\operatorname{sign}(\nabla\mu(p))$. This direction has $\|u(p)\|_\infty=1$ and satisfies
$\nabla\mu(p)\cdot u(p)=\|\nabla\mu(p)\|_1>0$ at every smooth boundary point. We fix a push distance $r>0$, set $x_0=p+r u(p)$, and verify at radius $\epsilon=r-\delta$. Equivalently, $r=\epsilon+\delta$. The sampled boundary point $p$ is then at $\ell_\infty$ distance $r$ from $x_0$, while the verification box $B_\epsilon(x_0)$ has buffer $\delta$ from this known boundary point.

\paragraph{Numerical separation check and verifier protocol}
\label{app:pnn-numerical}

The global condition $B_\epsilon(x_0)\cap D=\emptyset$ is the mathematical certificate for robustness. In the present experiments, we implement this check numerically before passing instances to the verifier. For each candidate center $x_0$, we search for an intersection between the verification box and the decision boundary by solving
$\min_{z\in B_\epsilon(x_0)} \mu(z)^2 .$
The minimum is zero if and only if the box intersects $D=\{\mu=0\}$. We solve this box-constrained problem using L-BFGS-B from $50$ random initializations in $B_\epsilon(x_0)$, with a maximum of $500$ iterations per restart. If the smallest value found is below $10^{-6}$, we treat the candidate as numerically intersecting the decision boundary and discard it. Otherwise, we keep the candidate as passing the numerical separation check.

For each retained instance, we export the network to ONNX and generate a VNNLIB specification for the $\ell_\infty$ box $B_\epsilon(x_0)$. The input constraints are $x_{0,i}-\epsilon\leq x_i\leq x_{0,i}+\epsilon$ for $i=1,\ldots,n$. The output constraint requires the predicted label at $x_0$ to remain unchanged throughout the box. Equivalently, if $\mu(x_0)>0$, the verifier is asked to prove $\mu(x)>0$ for all $x\in B_\epsilon(x_0)$, and if $\mu(x_0)<0$, it is asked to prove $\mu(x)<0$ for all $x\in B_\epsilon(x_0)$.

We record verifier outcomes as \emph{robust}, \emph{non-robust}, \emph{timeout}, or \emph{unknown}. A \emph{robust} result means that the verifier certified the output constraint over the full box. A \emph{non-robust} result means that it found a violation of the specification. A \emph{timeout} or \emph{unknown} result means that no certificate or counterexample was found within the time limit.

\paragraph{Ablation Study Results}
\label{app:pnn-results}

Table~\ref{tab:boundary-verification-results} summarizes an ablation study for the shallow polynomial network constructor using $\alpha,\beta$-CROWN. We vary one parameter at a time around a baseline configuration, keeping the push distance fixed at $r=0.02$ and setting the verification radius to $\epsilon=r-\delta$. Each row reports the number of retained instances submitted to $\alpha,\beta$-CROWN after the numerical separation check in Section~\ref{app:pnn-numerical}. A timeout is recorded when $\alpha,\beta$-CROWN does not return either a certificate or a counterexample within the time limit in the \emph{Limit} column.

The clearest transitions in this ablation occur when increasing the polynomial degree $d$ or decreasing the buffer $\delta$. This matches the geometry of the construction: $d$ controls the degree of the margin polynomial, and hence the algebraic complexity of the decision boundary. This is also consistent with the ED degree formulas in \cite{alexandr2026robustness}, which count complex critical points of the closest-boundary equations and involve powers of $d-1$. The buffer $\delta$ controls a complementary geometric difficulty: since $\epsilon=r-\delta$, smaller $\delta$ places the verification box closer to the sampled boundary point, leaving less separation for $\alpha,\beta$-CROWN to certify. In contrast, varying $n$ or $h$ changes the ambient dimension or the number of hidden units, but has a milder effect in this sweep.

\begin{table}[th]
\caption{Verification outcomes for retained near-boundary instances generated from shallow polynomial networks and evaluated with $\alpha,\beta$-CROWN. The push distance is fixed at $r=0.02$, and the verification radius is $\epsilon=r-\delta$. Mean time includes timeout runs; mean time without timeouts averages only certified runs.}

\setlength{\tabcolsep}{4pt}

\centering
\begin{tabular}{c c c c c c c c c c}
\toprule
$n$ & $d$ & $h$ & $\delta$ & Limit & Points & Robust & Timeouts & Mean time & Mean time (no T/O) \\
\midrule
50  & 10 & 100 & $5\cdot10^{-3}$ & 300s & 5 & 5 & 0 & 9.10s  & 9.10s \\
100 & 10 & 100 & $5\cdot10^{-3}$ & 300s & 5 & 5 & 0 & 7.88s  & 7.88s \\
200 & 10 & 100 & $5\cdot10^{-3}$ & 300s & 5 & 5 & 0 & 7.86s  & 7.86s \\
300 & 10 & 100 & $5\cdot10^{-3}$ & 300s & 5 & 5 & 0 & 8.42s  & 8.42s \\
\midrule
100 & 2  & 100 & $5\cdot10^{-3}$ & 300s & 5 & 5 & 0 & 8.32s  & 8.32s \\
100 & 4  & 100 & $5\cdot10^{-3}$ & 300s & 5 & 5 & 0 & 8.63s  & 8.63s \\
100 & 6  & 100 & $5\cdot10^{-3}$ & 300s & 5 & 5 & 0 & 8.36s  & 8.36s \\
100 & 8  & 100 & $5\cdot10^{-3}$ & 300s & 5 & 5 & 0 & 8.22s  & 8.22s \\
100 & 10 & 100 & $5\cdot10^{-3}$ & 300s & 5 & 5 & 0 & 7.88s  & 7.88s \\
100 & 12 & 100 & $5\cdot10^{-3}$ & 300s & 5 & 5 & 0 & 8.23s  & 8.23s \\
100 & 14 & 100 & $5\cdot10^{-3}$ & 300s & 5 & 5 & 0 & 8.22s  & 8.22s \\
100 & 16 & 100 & $5\cdot10^{-3}$ & 300s & 5 & 5 & 0 & 8.00s  & 8.00s \\
100 & 18 & 100 & $5\cdot10^{-3}$ & 300s & 5 & 4 & 1 & 68.43s & 8.53s \\
100 & 20 & 100 & $5\cdot10^{-3}$ & 300s & 5 & 3 & 2 & 129.13s & 9.50s \\
100 & 22 & 100 & $5\cdot10^{-3}$ & 300s & 5 & 2 & 3 & 187.51s & 8.85s \\
\midrule
100 & 10 & 50  & $5\cdot10^{-3}$ & 300s & 5 & 5 & 0 & 6.28s  & 6.28s \\
100 & 10 & 100 & $5\cdot10^{-3}$ & 300s & 5 & 5 & 0 & 7.88s  & 7.88s \\
100 & 10 & 200 & $5\cdot10^{-3}$ & 300s & 5 & 5 & 0 & 6.22s  & 6.22s \\
100 & 10 & 300 & $5\cdot10^{-3}$ & 300s & 5 & 5 & 0 & 6.12s  & 6.12s \\
100 & 10 & 500 & $5\cdot10^{-3}$ & 300s & 5 & 5 & 0 & 6.77s  & 6.77s \\
\midrule
100 & 10 & 100 & $10^{-3}$       & 300s & 5 & 4 & 1 & 66.72s & 6.58s \\
100 & 10 & 100 & $5\cdot10^{-3}$ & 300s & 5 & 5 & 0 & 7.88s  & 7.88s \\
100 & 10 & 100 & $10^{-4}$       & 300s & 5 & 1 & 4 & 248.96s & 7.94s \\
\bottomrule
\end{tabular}
\label{tab:boundary-verification-results}
\end{table}

Table~\ref{tab:VeriStress-boundary-verification-results} reports a complementary verifier-comparison experiment on selected polynomial-network instances from the VeriStress study. Unlike Table~\ref{tab:boundary-verification-results}, this table is not an ablation study for $\alpha,\beta$-CROWN over the full parameter sweep. Instead, each row records the outcome of running several verifiers on a retained near-boundary instance generated by the same constructor. The purpose is to show that the construction produces instances that can stress multiple verification back ends. Across these selected instances, $\alpha,\beta$-CROWN certifies many cases but begins to time out as $d$ increases or $\delta$ decreases, while NeuralSAT and PyRAT time out on more of the listed configurations.

\begin{table}[th]
\caption{Verifier outcomes for selected retained near-boundary instances used in the VeriStress study described in Section~\ref{sect:exp}. The push distance is fixed at $r=0.02$, the verification radius is $\epsilon=r-\delta$, and the timeout is $360$ seconds. Entries report certification time when the verifier certifies robustness and ``Timeout'' otherwise. Marabou and nnenum returned errors on these instances and are omitted.}
\centering
\begin{tabular}{c c c c c c c c c c}
\toprule
$n$ & $d$ & $h$ & $\delta$ 
& $\alpha,\beta$-CROWN outcome 
& NeuralSAT outcome 
& PyRAT outcome\\
\midrule
50  & 10 & 100 & $5\cdot10^{-3}$ & 9.81s & 39.08s & Timeout \\
100 & 2  & 100 & $5\cdot10^{-3}$ &  8.33s & Timeout & Timeout \\
100 & 6  & 100 & $5\cdot10^{-3}$ &  7.97s & 34.27s & Timeout \\
100 & 10 & 50  & $5\cdot10^{-3}$ &  7.79s & 37.74s & Timeout \\
100 & 10 & 500 & $5\cdot10^{-3}$ &  7.99s & 43.60s & Timeout \\
100 & 12 & 100 & $5\cdot10^{-3}$ &  7.61s & 37.17s & Timeout \\
100 & 14 & 100 & $5\cdot10^{-3}$ &  7.95s & 45.65s & Timeout \\
100 & 16 & 100 & $5\cdot10^{-3}$ &  8.09s & 38.13s & Timeout \\
100 & 17 & 100 & $5\cdot10^{-3}$ &  Timeout & Timeout & Timeout \\
100 & 18 & 100 & $5\cdot10^{-3}$ &  9.33s & 41.75s & Timeout \\
100 & 18 & 100 & $5\cdot10^{-3}$ &  9.89s & 42.29s & Timeout \\
100 & 19 & 100 & $5\cdot10^{-3}$ &  Timeout & Timeout & Timeout \\
100 & 19 & 100 & $5\cdot10^{-3}$ &  Timeout & Timeout & Timeout \\
100 & 20 & 100 & $5\cdot10^{-3}$ &  30.54s & 41.20s & Timeout \\
100 & 21 & 100 & $5\cdot10^{-3}$ &  Timeout & Timeout & Timeout \\
100 & 22 & 100 & $5\cdot10^{-3}$ &  34.43s & 44.97s & Timeout \\
100 & 10 & 100 & $2\cdot10^{-3}$ &  34.73s & Timeout & Timeout \\
100 & 10 & 100 & $1\cdot10^{-3}$ &  32.49s & Timeout & Timeout \\
100 & 10 & 100 & $5\cdot10^{-4}$ &  33.92s & Timeout & Timeout \\
100 & 10 & 100 & $2\cdot10^{-4}$ &  34.88s & Timeout & Timeout \\
100 & 10 & 100 & $1\cdot10^{-4}$ &  Timeout & Timeout & Timeout \\
100 & 18 & 100 & $1\cdot10^{-3}$ &  Timeout & Timeout & Timeout \\
\bottomrule
\end{tabular}
\label{tab:VeriStress-boundary-verification-results}
\end{table}

\section{Proof of Proposition~\ref{prop:beta_bound}}\label{app:Proofs}

Recall that
$L_c := \sup\limits_{x\in \mathcal{B}_\epsilon(x_0)}
\sup\limits_{\xi\in \partial_c \mu(x)} \|\xi\|_{p,*}$, where $\|\cdot\|_{p,*}$ is the dual norm of $\|\cdot\|_p$. In the smooth setting of Proposition~\ref{prop:beta_bound}, this reduces to
$L_c=\sup\limits_{x\in\mathcal{B}_\epsilon(x_0)}\|\nabla\mu(x)\|_{p,*}$. We assume the gradient is $\beta$-Lipschitz with respect to the same primal-dual norm pair, meaning
$\|\nabla\mu(x)-\nabla\mu(y)\|_{p,*}\leq \beta\|x-y\|_p$ for all
$x,y\in\mathcal{B}_\epsilon(x_0)$.

In the following proof, we cover the perturbation ball by smaller $\ell_p$ balls of radius $r$. On each small ball, the first-order Taylor approximation to $\mu$ is accurate up to order $\frac{\beta r^2}{2}$. Choosing $r=\sqrt{2\tau\epsilon L_c/\beta}$ makes this local approximation error at most $\tau\epsilon L_c$, and the number of required balls gives the stated affine-cover bound.

\begin{proof}
If $L_c=0$, then $\nabla\mu(x)=0$ for all $x\in\mathcal{B}_\epsilon(x_0)$, so $\mu$ is constant on the convex set $\mathcal{B}_\epsilon(x_0)$. Hence one affine function represents $\mu$ exactly, giving $N_{\mathrm{aff}}(\delta)=1$ and $A_\tau^\star=0$. The result is therefore immediate. 

Next, assume $L_c>0$ and let $\delta:=\tau\epsilon L_c$. If $\beta=0$, then $\nabla\mu$ is constant on $\mathcal{B}_\epsilon(x_0)$, so $\mu$ is affine on $\mathcal{B}_\epsilon(x_0)$. Again, $N_{\mathrm{aff}}(\delta)=1$ and $A_\tau^\star=0$. So, assume $\beta>0$ and set
\begin{equation}
    r:=
    \sqrt{\frac{2\delta}{\beta}}=
    \sqrt{\frac{2\tau\epsilon L_c}{\beta}} .
\end{equation}

Let $\mathcal{B}_p(z,r):=\{x:\|x-z\|_p\leq r\}$ be a closed $\ell_p$ ball. By the standard covering-number bound for an $\ell_p$ ball, there exist points $z_1,\dots,z_M\in\mathcal{B}_\epsilon(x_0)$ such that:
\begin{equation}
    \mathcal{B}_\epsilon(x_0)
    \subseteq
    \bigcup_{i=1}^M
    \left(
        \mathcal{B}_p(z_i,r)\cap \mathcal{B}_\epsilon(x_0)
    \right),
    \;\;\;\;\;
    M\leq\left(1+\frac{2\epsilon}{r}\right)^d .
\end{equation}
Define $U_i:=\mathcal{B}_p(z_i,r)\cap \mathcal{B}_\epsilon(x_0)$. Each $U_i$ is convex because it is the intersection of two convex sets. For each cell $U_i$, define the first-order affine approximation
\begin{equation}
    \ell_i(x):=\mu(z_i)+\nabla\mu(z_i)^\top(x-z_i).
\end{equation}
Fix any $x\in U_i$. By the fundamental theorem of calculus along the segment from $z_i$ to $x$, we have that:
\begin{align}
    \mu(x)-\mu(z_i) &= \int_0^1 \nabla\mu\bigl(z_i+t(x-z_i)\bigr)^\top (x-z_i)
    \,dt, \\
    \mu(x)-\ell_i(x) &= \int_0^1\left[
        \nabla\mu\bigl(z_i+t(x-z_i)\bigr)
        -
        \nabla\mu(z_i)\right]^\top(x-z_i)\,dt .
\end{align}
Using Holder duality and the $\beta$-Lipschitz property of the gradient:
\begin{align}
    |\mu(x)-\ell_i(x)| &\leq \int_0^1 \left\|
        \nabla\mu\bigl(z_i+t(x-z_i)\bigr) - \nabla\mu(z_i) \right\|_{p,*} \|x-z_i\|_p\,dt \\
    &\leq\int_0^1\beta t\|x-z_i\|_p^2\,dt \\
    &=\frac{\beta}{2}\|x-z_i\|_p^2 .
\end{align}
Since $x\in U_i\subseteq \mathcal{B}_p(z_i,r)$, we have $\|x-z_i\|_p\leq r$, and hence
$|\mu(x)-\ell_i(x)|\leq \frac{\beta}{2}r^2=\delta$. Thus each $U_i$ admits an affine approximation with error at most $\delta$, so:
\begin{equation}
    N_{\mathrm{aff}}(\delta)
    \leq
    M
    \leq
    \left(
        1+\frac{2\epsilon}{r}
    \right)^d
\end{equation}
Substituting the definition of $r$ gives:
\begin{equation}
    \frac{2\epsilon}{r}
    = 2\epsilon \sqrt{\frac{\beta}{2\tau\epsilon L_c}}
    =\sqrt{\frac{2\beta\epsilon}{\tau L_c}}
    = \sqrt{\frac{2\widetilde\beta}{\tau}},\qquad
    \widetilde\beta:=\frac{\beta\epsilon}{L_c}.
\end{equation}
Therefore:
\begin{align}
    N_{\mathrm{aff}}(\tau\epsilon L_c)
    \leq
    \left(
        1+\sqrt{\frac{2\widetilde\beta}{\tau}}
    \right)^d \\
    \implies A_\tau^\star
     =
    \log N_{\mathrm{aff}}(\tau\epsilon L_c)\leq d\log
    \left(
        1+\sqrt{\frac{2\widetilde\beta}{\tau}}
    \right)
\end{align}
\end{proof}

The proposition shows that in smooth networks, affine-cover complexity is controlled by the normalized gradient-variation scale $\widetilde\beta=\frac{\beta\epsilon}{L_c}$. In principle, one could estimate $\widetilde\beta$ directly and use the bound as a profile component. In practice, however, estimating a global gradient-Lipschitz constant over $\mathcal{B}_\epsilon(x_0)$ is often conservative, so Section~\ref{sect:profiles} instead uses the empirical statistic $A_\tau$, which samples and counts distinct local linearization classes. For piecewise-linear networks where a global smoothness constant may be infinite, $A_\tau$ is interpreted as an empirical count of distinct local affine behaviors encountered inside the perturbation set.

\section{Verifier Taxonomy by Certificate Type}\label{app:taxonomy}

Our characterization of instance difficulty is guided by the wide array of verification algorithm implementations, as the profile must apply to all verifiers. Verification algorithms differ not only in implementation but in the fundamental mathematical object they construct as evidence that $\min\limits_{x \in \mathcal{B}_\epsilon(x_0)} \mu(x) > 0$. We organize the verifier landscape by certificate type rather than by implementation details such as solver backend or bounding subroutine. We identify four certificate types:

Style D (Dual Relaxation). The verifier produces a dual certificate, i.e., a feasible solution to the dual of a convex relaxation of $\min\limits_{x \in \mathcal{B}} \mu(x)$. Verification terminates in a single pass and no branching or combinatorial search is performed. The quality of the certificate depends entirely on the tightness of the relaxation. Examples: IBP \cite{DBLP:journals/corr/abs-1810-12715}, CROWN \cite{zhang2018efficientneuralnetworkrobustness}, $\alpha$-CROWN \cite{
xu2020automatic, 
xu2021fastcompleteenablingcomplete} without branching.

Style P (Partition Search). The verifier produces a finite cover $\{U_i\}$ of $\mathcal{B}_\epsilon(x_0)$ together with a bound on $\mu$ over each $U_i$, jointly witnessing $\mu > 0$ on $\mathcal{B}_\epsilon(x_0)$. The cover is constructed by recursively splitting the domain and bounding $\mu$ on each piece via relaxation, pruning subproblems whose relaxation already certifies the property. The cost depends on the number of subproblems solved via pruning. Examples: $\alpha,\beta$-CROWN, Marabou's divide-and-conquer mode \cite{10.1007/978-3-030-25540-4_26}.

Style S (SAT Refutation). The verifier produces a refutation certificate, i.e. a proof that no assignment of  activation states is jointly consistent with the constraints implied by $\mu(x) \leq 0$. The search is often conducted by a DPLL(T)-style SAT engine \cite{10.1007/978-3-540-27813-9_14} over Boolean activation variables with a solver checking consistency of each partial assignment and learned conflict clauses pruning future candidates. The cost depends on how many partial assignments must be explored before the search space is exhausted. Examples: NeuralSAT \cite{duong2024dplltframeworkverifyingdeep}, Marabou's DPLL(T) engine. 

Style R (Reachability). The verifier produces a reachability certificate, i.e. an overapproximation of the output range $\{f(x) : x \in \mathcal{B}_\epsilon(x_0)\}$, obtained by propagating an abstract set representation of $\mathcal{B}_\epsilon(x_0)$ forward through $f$ layer by layer. Unlike other verifier styles, the certificate is typically constructed in the forward direction without explicit duality or domain decomposition, and its cost depends on how the abstract representation grows as it passes through the network $f$. Examples: nnenum (zonotopes with refinement) \cite{nnenum}, NNV (star sets) \cite{tran2020nnvneuralnetworkverification}.

Note that several verifiers combine multiple certificate types. For instance, $\alpha,\beta$-CROWN uses a dual relaxation certificate at each node of a partition search (styles D and P). Similarly, nnenum uses abstract propagation with domain splitting (styles P and R).

\section{Bug Discovery in Open-Source Verifier}\label{app:NeuralSAT_Bug}

During development of VeriStress-GT, we reported a bug in the open-source verifier [Verifier]\footnote{Identity of codebase has been redacted to preserve anonymity in the double blind review process.}. The issue arose in [Verifier]'s MIP presolve routine for specifications with disjunctive output constraints. In the verifier convention used here, SAT means that the adversarial constraints are feasible, i.e., a counterexample exists, while UNSAT means that no counterexample exists for the given specification.

The relevant class of specifications can be written as a disjunction$\Phi = \bigvee_{j=1}^{m} \Phi_j$. 
Correct handling of such a specification requires OR semantics, i.e. the verifier should return SAT if any disjunct $\Phi_j$ is feasible, should return UNSAT only if all disjuncts are infeasible, and should return UNKNOWN if no disjunct is shown feasible but at least one disjunct cannot be resolved. However, the observed behavior in the MIP presolve routine appeared to return the status of the final disjunct checked. As a result, if an earlier disjunct was feasible but the final disjunct was infeasible, the routine could incorrectly return UNSAT on a satisfiable adversarial specification.

This failure mode occurred only under a specific combination of conditions. The specification contained multiple disjuncts, at least one non-final disjunct was feasible, the final disjunct in iteration order was infeasible, and MIP presolve was enabled. The reported fix was to aggregate disjunct statuses explicitly. In other words, return SAT immediately upon finding a feasible disjunct and otherwise return UNKNOWN if any disjunct was unresolved, and return UNSAT only after all disjuncts have been proved infeasible.

This discovery illustrates one of the motivations for VeriStress-GT. Benchmarks whose labels are inferred from verifier consensus or from the absence of counterexamples can obscure implementation-level soundness failures. By contrast, VeriStress-GT instances have verifier-independent ground-truth robustness labels, allowing incorrect verifier outcomes to be detected directly rather than treated as ambiguous disagreements among tools. We appreciate the responsiveness of the developers of the [Verifier] package in quickly reviewing and approving our proposed fix for this issue. 

\section{Difficulty Profile Study Details}\label{app:comp_detail_results}

\subsection{Design Principles}\label{app:design}
Below we state the three major design principles used to guide the process of Difficulty Profile component selection:

(D1, Instance-level computability): Each component can be computed without access to the internals of a particular verifier. Thus the difficulty profile component characterizes the problem instance rather than implementation and efficiency details of a verifier.

(D2, Certificate Coverage): The components are selected to cover varying bottlenecks in the certification problem. These correspond to fundamentally different reasons a verifier may require branching or refinement and help ensure the Difficulty Profile is applicable to all verifier types (See Appendix~\ref{app:taxonomy} for a description of different verifier styles).

(D3, Non-Redundancy and Interpretability): No two components intend to capture the same phenomenon. Otherwise, one may saturate the pool of possible components with many statistics recorded during experimentation, reducing interpretability and leading to a bloated set of components. 

\subsection{Component Search}\label{app:search}
Lastly, we note that many difficulty profile components were tested before decided on the canonical set enumerated in Section \ref{sect:profiles}. Table~\ref{tab:discarded_difficulty_components} below details many of the component ideas that were tested but ultimately not included, typically due to limited correlations with runtime or timeout predictions.

\begin{table}[h]
\caption{Difficulty Profile component ideas tested and corresponding explanations as to why they did not appear in the canonical profile defined in Section~\ref{sect:profiles}.}
\centering
\scriptsize
\setlength{\tabcolsep}{4pt}
\renewcommand{\arraystretch}{1.15}
\begin{tabularx}{\textwidth}{
@{}
>{\raggedright\arraybackslash}p{0.22\textwidth}
>{\raggedright\arraybackslash}p{0.25\textwidth}
>{\raggedright\arraybackslash}X
@{}}
\hline
\textbf{Component Idea} & \textbf{Intended Signal} & \textbf{Reason Discarded}\\
\hline\hline

Empirical mean margin
& Average robustness slack over sampled perturbations.
& Discarded in favor of $\widehat M_{\min}$, which better matches verification failure: a single low-margin region can dominate robustness difficulty even when the average margin is large. \\
\hline

IBP absolute gap
& Raw looseness of interval propagation.
& Folded into $G_{\mathrm{IBP}}$. The relative version was more comparable across constructions because raw gaps are strongly affected by logit scale. \\
\hline

IBP lower margin
& Whether plain IBP certifies the instance.
& Used as a diagnostic, but discarded as a profile component because $G_{\mathrm{IBP}}$ captures relaxation looseness more directly and is measured via a continuous scale \\
\hline

Unstable count
& Number of ambiguous nonlinear units.
& Folded into $U$. The fraction normalizes for network size \\
\hline

Total nonlinear units
& Raw nonlinear network size.
& Discarded because it is mostly a size statistic. $U$ and $A_\tau$ better capture whether nonlinearities are actually difficult inside the perturbation set. \\
\hline

Raw local pattern count
& Number of observed local activation patterns.
& Folded into $A_\tau$. The tolerance-based/log-scaled version was more stable \\
\hline

Gradient dimension maximum
& Worst sampled gradient-dimensionality estimate.
& Discarded in favor of $d_{\mathrm{eff}}$ mean/aggregate, since the maximum was noisier and more sensitive to outlier samples. \\
\hline

Gradient norm
& Local first-order sensitivity scale.
& Discarded due to scale dependence and limited standalone intuition. Margin scale and relaxation looseness were better captured by $\widehat M_{\min}$ and $G_{\mathrm{IBP}}$. \\
\hline

Gradient sensitivity
& Variation of gradients across the perturbation set.
& Discarded because it was noisier than the canonical components. Its useful signal was mostly absorbed by $A_\tau$ and $d_{\mathrm{eff}}$. \\
\hline

Gradient concentration
& Whether sensitivity lies in few coordinates.
& Folded into $d_{\mathrm{eff}}$, which gives a more direct and normalized measure of effective sensitivity dimension. \\
\hline

Relative linearization error
& Deviation from first-order local behavior.
& Discarded due to weaker and less consistent correlation. $A_\tau$ gave a cleaner proxy for local piecewise-linear complexity. \\
\hline

Absolute linearization error
& Raw first-order approximation error.
& Discarded because it is scale-sensitive. The relative version was preferable, but ultimately $A_\tau$ had clearer intuition. \\
\hline

First-order margin ratio
& Margin compared to local linear drift.
& Discarded because it mixed multiple effects. $\widehat M_{\min}$ and $d_{\mathrm{eff}}$ separated margin tightness from directional complexity more cleanly. \\
\hline

Lipschitz-margin ratio
& Worst-case sensitivity relative to margin.
& Discarded because global/local Lipschitz estimates were loose and scale-dependent. $\widehat M_{\min}$ and $G_{\mathrm{IBP}}$ were more directly tied to verification outcomes. \\
\hline

Estimated local Lipschitz
& Sampled local sensitivity bound.
& Discarded due to estimator noise and limited added value beyond gradient- and margin-based canonical components. \\
\hline

Curvature proxy
& Smooth nonlinear bending of the margin.
& Discarded because estimates were noisy and less natural for ReLU networks. $A_\tau$ was a better architecture-agnostic nonlinearity proxy. \\
\hline

Parameter count
& Raw model size.
& Does not distinguish easy large networks from hard small ones. \\
\hline

Layer count
& Raw network depth.
& Discarded as too architecture-dependent. Depth can create difficulty, but only through effects better captured by $U$, $A_\tau$, $G_{\mathrm{IBP}}$, or $d_{\mathrm{eff}}$. \\
\hline

Instability-region interaction
& Many unstable units across many local regions.
& Discarded because it was less interpretable than reporting $U$ and $A_\tau$ separately, with limited evidence of consistent added correlation. \\
\hline

Looseness-instability interaction
& Loose relaxations plus many ambiguous units.
& Discarded because it was mostly redundant with the two canonical components and harder to explain cleanly. \\
\hline

Looseness-region interaction
& Loose relaxations across many regions.
& Discarded because it added complexity without enough interpretability gain over $G_{\mathrm{IBP}}$ and $A_\tau$. \\
\hline

Margin-looseness interaction
& Small margin plus loose bounds.
& Discarded because it was highly sensitive when $\widehat M_{\min}$ was small. Reporting $\widehat M_{\min}$ and $G_{\mathrm{IBP}}$ separately was more stable. \\
\hline

\end{tabularx}
\label{tab:discarded_difficulty_components}
\end{table}

\section{Numerical Study Details}\label{app:detail_results}
Table~\ref{tab:verifier_outcomes} displays the aggregated VeriStress-GT stress-test study results, while Table~\ref{tab:hyperparams} details what hyperparameter ranges were used in each construction for the instantiation of the VeriStress-GT framework. To see the specific configuration for all $225$ instances for the stress-test, see ``veristress.yaml'' and ``polynomial$\_$stress$\_22.$yaml'' in the configs folder in our public code repository: \url{https://github.com/dtroxell19/VeriStressGT.git}

\begin{table}[ht]
\centering
\begin{tabular}{lrrrrr}\label{tab:verifier_outcomes}
\textbf{Verifier} & \textbf{UNSAT} & \textbf{SAT} & \textbf{Timeout/Unknown} & \textbf{Error/Unsupported} & \textbf{Total} \\
\hline
$\alpha,\beta$-CROWN   & 193 & 0 & 32 & 0  & 225 \\
Marabou  & 122 & 5 & 41 & 57 & 225 \\
NeuralSAT & 188 & 2 & 23 & 12  & 225 \\
PyRAT     & 160 & 0 & 25 & 40 & 225 \\
nnenum     & 102 & 1 & 45 & 77 & 225 \\
\hline
\end{tabular}
\caption{Verification outcomes from Section~\ref{sect:exp}, aggregated across constructors.}
\label{tab:verifier_outcomes}
\end{table}

\begin{table}[ht]
\caption{Constructor hyperparameter ranges for the VeriStress-GT instantiation described in Section~\ref{sect:exp}. The fixed column lists compact constructor-level constants, while the varied column reports the hyperparameters swept across instances. Note that $d$ represents input dimension and $C$ represents number of output classes.}
\centering
\scriptsize
\setlength{\tabcolsep}{3pt}
\renewcommand{\arraystretch}{1.18}
\begin{tabularx}{\textwidth}{
@{}
>{\raggedright\arraybackslash}p{0.145\textwidth}
r
>{\raggedright\arraybackslash}p{0.095\textwidth}
>{\raggedright\arraybackslash}p{0.315\textwidth}
>{\raggedright\arraybackslash}X
@{}}
\hline
Constructor & \# & Fixed & Varied Hyperparameters & Hyperparameter Descriptions \\
\hline

MEAP
& 14
& $d=100$, $C=10$
& $P\in\{2,16,32,128\}$; $\varepsilon\in\{0.05,0.1,0.25,0.5\}$; $\gamma\in\{10^{-5},10^{-4},10^{-3},1,100\}$
& $P$: ReLU pair count; $\varepsilon$: perturbation radius; $\gamma$: certified margin buffer. \\
\hline

MILP Exact-Radius
& 31
& $d=50$, $C=5$
& $h_1,h_2\in\{5,10,20,50,100\}$; $\varepsilon_{\mathrm{frac}}\in\{0.5,0.9,0.99,0.999,0.9999,0.99999\}$
& $h_1,h_2$: hidden-layer widths; $\varepsilon_{\mathrm{frac}}$: 
Exact radius multiplier. \\
\hline

Input-Corner Stress
& 22
& $d=100$, $C=10$, $\varepsilon=0.2$
& $\texttt{hinge\_l1}\in\{10^{-4},10^{-2},1,10^2,10^4,10^6\}$; $\texttt{num\_hinges}\in\{16,256,1024,4096\}$
& $\texttt{hinge\_l1}$: hinge slope scale; $\texttt{num\_hinges}$: convex hinge count. \\
\hline

Constant-on-Box Embedding
& 5
& $d=100$, $C=10$, $\varepsilon=0.1$
& $\texttt{wide-layers}\in\{1,5,10,20,50\}$
& $\texttt{wide-layers}$: downstream network depth. \\
\hline

Deep Contractive CNN
& 50
& $\varepsilon=0.02$, $C=10$, $k=3$
& $\texttt{depth}\in\{2,4,6,8,10\}$; $\texttt{channels}\in\{4,8,16,32,64\}$; $\lambda\in\{0.70,0.80,0.90,0.95,0.98\}$; $\texttt{margin}\in\{10^{-4},5\cdot10^{-4},10^{-3},10^{-2},10^{-1}\}$; $\texttt{bias\_instability\_frac}\in\{0.10,0.25,0.40,0.55,0.70\}$
& $\texttt{depth}$: contractive block count; $\texttt{channels}$: feature-map width; $\lambda$: per-layer contraction; $\texttt{margin}$: certified slack floor; $\texttt{bias\_instability\_frac}$: instability placement target. \\
\hline

Paired-Biases CNN
& 46
& $\varepsilon=0.05$, $C=10$, $k=3$
& $P\in\{2,4,8,16,32\}$; $\texttt{margin}\in\{10^{-3},0.05,0.1,0.5,1.0\}$; $\delta\in\{0.005,0.01,0.025,0.05,0.1\}$; $\texttt{num\_backbone\_layers}\in\{1,2,3,4\}$; $H,W\in\{4,8,12,16\}$
& $P$: paired-channel count; $\texttt{margin}$: global logit offset; $\delta$: bias gap half-width; $\texttt{num\_backbone\_layers}$: upstream CNN depth; $H,W$: spatial feature size. \\
\hline

Dominant-Key (Linear) Attention
& 18
& $C=10$, $d_k=n$
& $n\in\{2,4,6\}$; $d\in\{2,4,6,8,12\}$; $d_v\in\{2,4,8\}$; $\varepsilon\in\{0.02,0.05,0.08\}$; $\texttt{margin\_factor}\in\{1.1,1.5,2.0,3.0\}$; $(\texttt{gate\_scale},\texttt{noise\_scale})\in\{(0.30,0.10),(0.50,0.20),(0.80,0.30)\}$
& $n$: sequence length; $d$: token dimension; $d_v$: value dimension; $\varepsilon$: perturbation radius; $\texttt{margin\_factor}$: certificate headroom; $\texttt{gate\_scale}$: dominant-key strength; $\texttt{noise\_scale}$: off-key suppression. \\
\hline

Fixed-Ordering Attention
& 17
& $d_v=8$, $C=50$
& $n\in\{2,8,16,32\}$; $d\in\{4,8,16\}$; $\alpha\in\{2.0,5.0,10.0\}$; $\varepsilon\in\{5\cdot10^{-4},10^{-3},10^{-2}\}$; $\texttt{margin\_slack}\in\{1.00001,1.001,1.05,1.1,2.0,5.0\}$
& $n$: sequence length; $d$: token dimension; $\alpha$: attention score scale; $\varepsilon$: perturbation radius; $\texttt{margin\_slack}$: certificate headroom. \\
\hline
Polynomial Network
& 22
& C = 2
& $d\in\{50,100\}$; $\texttt{degree}\in [2,22]$; $h\in\{50,100,500\}$; $\delta\in[1\times10^{-4}, 5\times10^{-3}]$
& $d$: input dimension; $\texttt{degree}$: polynomial degree; $h$: hidden width; $\delta$: boundary buffer. \\
\hline

\end{tabularx}
\label{tab:hyperparams}
\end{table}

\begin{figure}[ht]
    \centering
    \includegraphics[width=.8\linewidth]{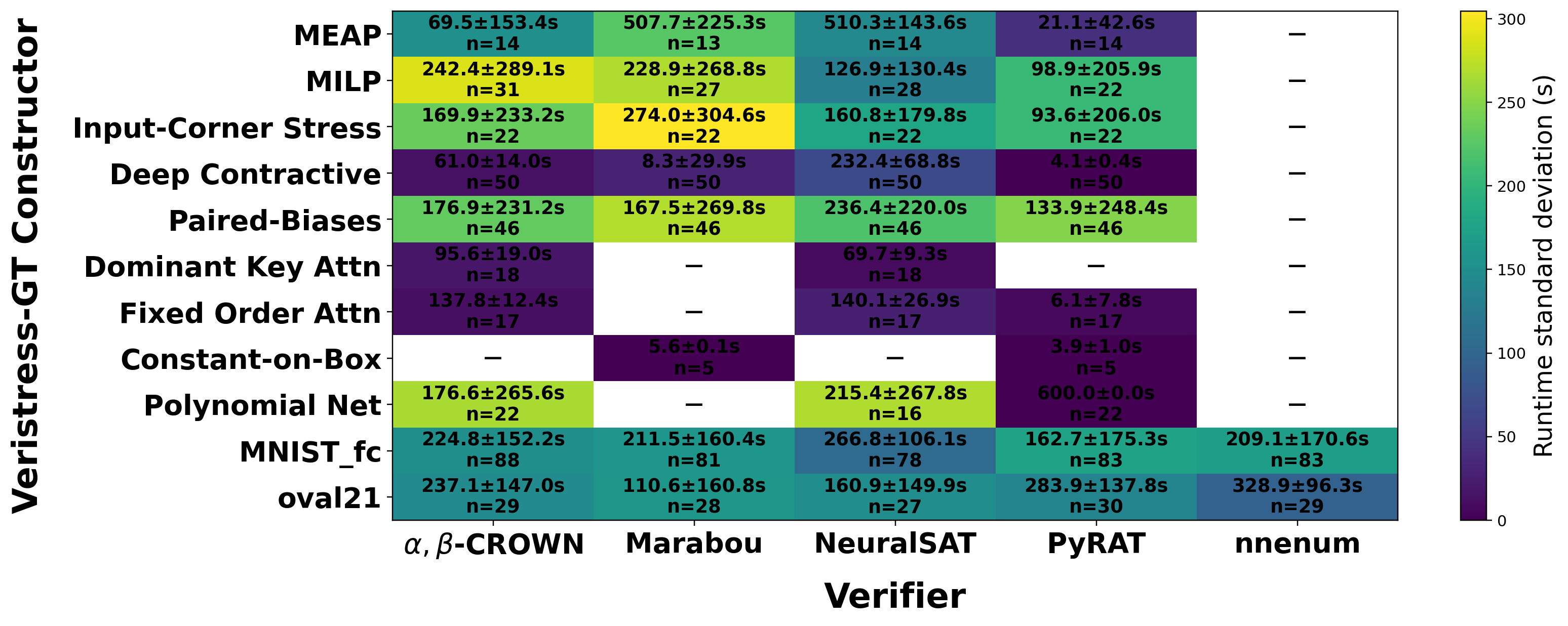}
    \caption{Runtime variability across instances for each constructor/benchmark and verifier. Each cell reports empirical mean $\pm$ standard deviation of runtime in seconds, computed over verified instances and timeouts using the same timing convention as Figure~\ref{fig:veriStress_results}. Verified instances use measured wall-clock time, while timeouts are assigned the benchmark timeout value. Error, unsupported, and unknown outcomes are excluded from the runtime-spread estimate. A cell is left blank if the standard deviation for the runtime was not applicable or $0$.
}
    \label{fig:var}
\end{figure}

Additionally, we also include trajectory plots showing Difficulty Profile component estimates and corresponding verifier outcomes for each instance in the external MNIST$\_$fc and oval$21$ benchmarks in Figures \ref{fig:mnist_traj} and \ref{fig:oval_traj}. To access the full csv files containing the difficulty profile estimates and corresponding outcomes (including those for the $225$ tested VeriStress-GT framework) please see the Results folder in our public code repository: \url{https://github.com/dtroxell19/VeriStressGT.git}

\begin{figure}[ht]
    \centering
    \includegraphics[width=.75\linewidth]{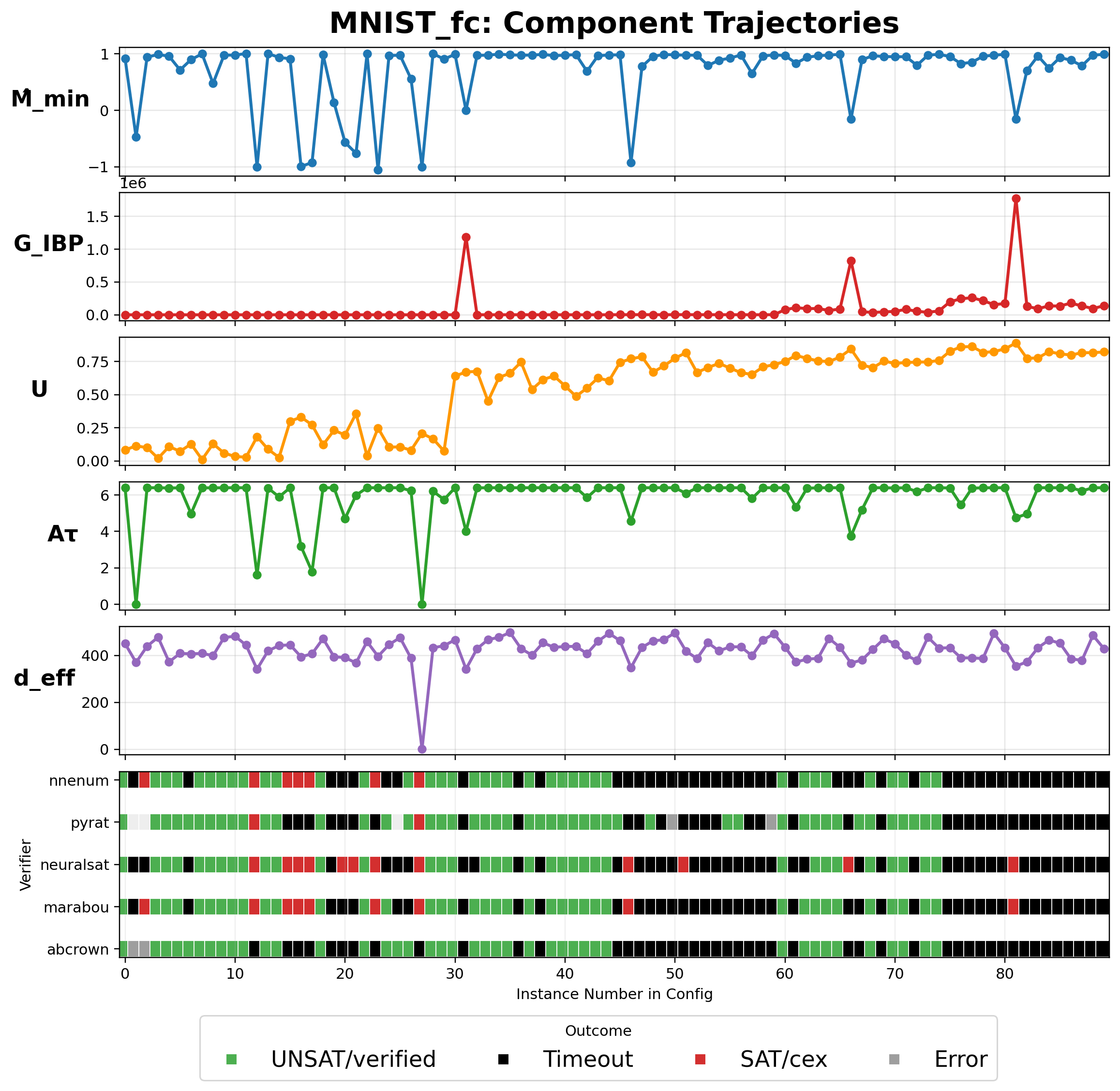}
    \caption{Difficulty profile component estimates and verifier outcomes for each of the $90$ instances in the MNIST$\_$fc VNNCOMP benchmark. Each instance had a timeout limit of $360$ seconds per~verifier.}
    \label{fig:mnist_traj}
\end{figure}

\begin{figure}[ht]
    \centering
    \includegraphics[width=.75\linewidth]{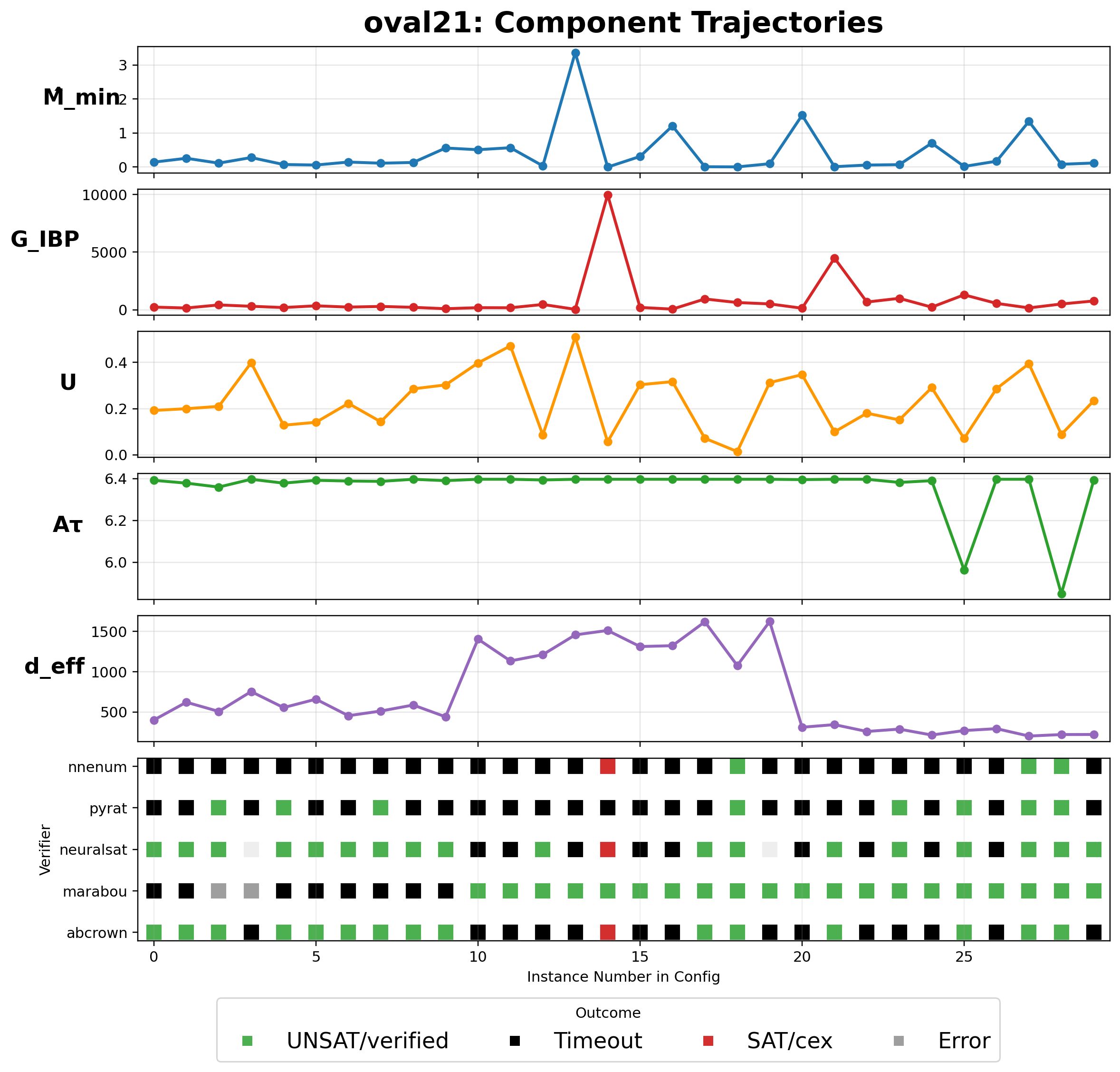}
    \caption{Difficulty profile component estimates and verifier outcomes for each of the $90$ instances in the Oval$21$ VNNCOMP benchmark. Each instance had a timeout limit of $360$ seconds per~verifier.}
    \label{fig:oval_traj}
\end{figure}

Finally, we also report the distribution of runtimes in Figure~\ref{fig:var}, corresponding to the reporting of averages in Figure~\ref{fig:veriStress_results}.

\end{document}